\documentclass{article}
\usepackage[utf8]{inputenc} % allow utf-8 input
\usepackage[T1]{fontenc}    % use 8-bit T1 fonts
\usepackage{fullpage}

\usepackage[round]{natbib}

\usepackage[margin = 1in]{geometry}

\usepackage{amsmath, amssymb, amsfonts}
\usepackage{amsthm, thmtools, thm-restate}
\usepackage{mathrsfs}
\usepackage{mathtools, bm, enumitem}
\usepackage{algorithm, algorithmicx, algpseudocode}
\usepackage{color}
\usepackage[table]{xcolor}  %
\usepackage{dsfont}
\usepackage{booktabs}

\newtheorem{theorem}{Theorem}%[section]
\newtheorem{proposition}[theorem]{Proposition}
\newtheorem{lemma}[theorem]{Lemma}

\newtheorem{definition}{Definition}
\newtheorem{assumption}[definition]{Assumption}

\newtheorem{remark}{Remark}

\newtheorem{example}[remark]{Example}

\definecolor{puorange}{rgb}{0.80,0.20,0}
\definecolor{bluegray}{rgb}{0.04,0,0.7}
\definecolor{greengray}{rgb}{0.05,0.50,0.15}
\definecolor{darkbrown}{rgb}{0.40,0.2,0.05}
\definecolor{darkcyan}{rgb}{0,0.4,1}
\definecolor{black}{rgb}{0,0,0}
\definecolor{grey}{rgb}{0.93,0.93,0.93}
\definecolor{royalazure}{rgb}{0.0, 0.22, 0.66}

\newcommand{\tabemph}[1]{\cellcolor{grey!30}\textcolor{black!50!royalazure}{#1}}%

\usepackage[colorlinks,citecolor=bluegray,linkcolor=darkbrown,urlcolor=blue,breaklinks]{hyperref}
\usepackage[capitalise]{cleveref}

\Crefname{assumption}{Assumption}{Assumption}

\newcommand{\reals}{{\mathbb R}}

\newcommand{\abs}[1]{\left| #1 \right|}
\newcommand{\norm}[1]{\left\lVert #1 \right\rVert}

\newcommand{\Vect}{\operatorname{Vec}}

\newcommand{\ip}[1]{{\langle #1 \rangle}}
\newcommand{\lone}{\mathbf{L}^1}
\newcommand{\ltwo}{\mathbf{L}^2}

\newcommand{\Expect}{\operatorname{\mathbb E}}

\newcommand{\Prob}{\operatorname{\mathbb P}}

\newcommand{\pspace}{\mathcal{M}_1}

\newcommand{\hnull}{\mathbf{H}_0}
\newcommand{\halt}{\mathbf{H}_1}

\newcommand{\calC}{\mathcal{C}}

\newcommand{\calF}{\mathcal{F}}

\newcommand{\calN}{\mathcal{N}}

\newcommand{\calX}{\mathcal{X}}
\newcommand{\calY}{\mathcal{Y}}
\newcommand{\calZ}{\mathcal{Z}}

\newcommand{\bbN}{{\mathbb N}}

\newcommand{\ind}{\mathds{1}}

\newcommand{\txtover}[2]{\overset{\mbox{\scriptsize #1}}{#2}}
\newcommand{\teotic}{{ETIC}}
\newcommand{\eoticrf}{{ETIC-RF}}
\newcommand{\eotic}{T}
\newcommand{\heotic}{T_n}
\newcommand{\px}{P_X}
\newcommand{\hpx}{\hat P_X}
\newcommand{\py}{P_Y}
\newcommand{\hpy}{\hat P_Y}
\newcommand{\pxy}{P_{XY}}
\newcommand{\hpxy}{\hat P_{XY}}
\newcommand{\prodxy}{\px \otimes \py}
\newcommand{\hprodxy}{\hpx \otimes \hpy}

\newcommand{\kl}[2]{\operatorname{KL}(#1 \Vert #2)}
\newcommand{\ot}{\operatorname{OT}}
\newcommand{\hsic}{\operatorname{HSIC}}
\newcommand{\sdiv}{\bar S}

\title{Entropy Regularized Optimal Transport Independence Criterion}
\author{Lang Liu$^1$ \qquad Soumik Pal$^2$ \qquad Zaid Harchaoui$^1$ \vspace{0.3cm} \\
$^1$ Department of Statistics, University of Washington \\
$^2$ Department of Mathematics, University of Washington}
\date{}

\begin{document}

\maketitle

\begin{abstract}
  We introduce an independence criterion based on entropy regularized optimal transport. Our criterion can be used to test for independence between two samples. We establish non-asymptotic bounds for our test statistic and study its statistical behavior under both the null hypothesis and the alternative hypothesis. The theoretical results involve tools from U-process theory and optimal transport theory. We also offer a random feature type approximation for large-scale problems, as well as a differentiable program implementation for deep learning applications.  We present experimental results on existing benchmarks for independence testing, illustrating the interest of the proposed criterion to capture both linear and nonlinear dependencies in synthetic data and real data.
\end{abstract}

\section{Introduction}
\label{sec:intro}
Statistical independence measures have been widely used in machine learning and statistics, ranging from independent component analysis~\citep{bach2002kernel,gretton2005kernel} to causal inference~\citep{pfister2018kernel,chakraborty2019distance}, and recently in self-supervised learning~\citep{li2021selfsupervised} and representation learning~\citep{ozair2019wasserstein}.
Classical dependence measures such as Pearson's correlation coefficient, Spearman's $\rho$, and Kendall's $\tau$~\citep{hoeffding1948nonparametric,kruskal1958ordinal,lehmann1966some} focus on real-valued one dimensional random variables and thus are not suitable for high dimensional data; see also~\citep{schweizer1981nonparametric,nikitin1995asymptotic}.

One popular choice of independence measures in high dimension is the Hilbert-Schmidt independence criterion (HSIC)~\citep{gretton2005kernel}.
This criterion was used to test independence by \citet{gretton2007kernel}.
Several extensions of HSIC are available, such as a relative dependency measure~\citep{bounliphone2015low} and a joint independence measure among multiple random elements~\citep{pfister2018kernel}.
Another choice is the distance covariance (dCov) of~\citet{szekely2007measuring}.
dCov was originally developed in Euclidean spaces using characteristic functions and later generalized to metric spaces~\citep{lyons2013distance}.
In fact, in their most general form, HSIC and dCov are equivalent as shown by~\citet{sejdinovic2013equivalence}.

A different line of research explored optimal transport to measure dependence.
The Wasserstein distance naturally defines a dependence measure when it is used to quantify the dissimilarity between the joint distribution and the product of marginals; see, e.g., \citep{cifarelli2017gini}.
The normalized version---the so-called Wasserstein correlation coefficient---has recently gained attention in \citep{wiesel2021measuring,mordant2021measuring,nies2021transport}.
Following the classical rank-based tests such as Pearson's $\rho$, optimal transport is also used to define multivariate ranks and the subsequent independence tests~\citep{shi2020distribution,deb2021multivariate}.
However, these tests can suffer from the curse of dimensionality or high computational complexity, limiting their practical usefulness; see~\citep{peyre2019computational} for a discussion.

A remedy to this challenge is to use the entropy regularized formulation of optimal transport.
This is particularly attractive from both a computational viewpoint~\citep{cuturi2013sinkhorn} and a statistical viewpoint~\citep{rigollet2018entropic}.
Moreover the empirical counterpart of entropy regularized optimal transport enjoys as an estimator a parametric rate of convergence and thus appears to overcome the curse of dimensionality~\citep{genevay2019sample,mena2019statistical}. The Sinkhorn divergence~\citep{feydy2019interpolating}, its centered version, defines a semi-metric on probability measures which metrizes weak convergence.
\citet{ramadas2017wasserstein} used it for two-sample testing and \citet{genevay2018learning} for generative modeling; see also \citep{salimans2018improving,sanjabi2018convergence}. Our approach to independence testing shall build upon this recent progress on entropic regularization.

\paragraph{Outline.}
In \Cref{sec:eotic}, we introduce the entropy regularized optimal transport independence criterion (\teotic) and discuss its key properties.
We propose the Tensor Sinkhorn algorithm with a random feature approximation to compute \teotic, which admits a quadratic scaling in time and space. We also show how to approximate \teotic~using random features, and how to differentiate through \teotic~in a framework of differentiable programming.
In \Cref{sec:thm}, we give our main theoretical results, i.e., non-asymptotic bounds, characterizing the statistical behavior of the empirical estimator of \teotic~under both the null and alternative hypotheses.
In \Cref{sec:experiments}, we compare the empirical behavior of \teotic~with HSIC on both synthetic and real data. The Appendix contains detailed proofs as well as additional experiments.

\paragraph{Related Work.}
Statistical metrics on the space of probability measures form the backbone of many dependence measures.
On the machine learning side, distributions are compared by embedding them into reproducing kernel Hilbert spaces \citep{gretton2007twosample,gretton2012kernel}.
The Hilbert-Schmidt independence criterion (HSIC) uses Hilbertian embeddings of probability distributions to compare the joint distribution and the product of marginals~\citep{gretton2005kernel,gretton2007kernel}.
On the statistics side, distributions defined on Euclidean spaces are compared via their characteristic functions, leading to the so-called energy distance~\citep{szekely2004testing}.

A closely related dependence measure is the distance covariance \citep{szekely2007measuring}.
These distances were later generalized to general metric spaces of negative type by~\citet{lyons2013distance}, unifying the two notions via the Barycenter map---a quantity similar to the feature map in kernel methods.
In fact, the kernel-based and distance-based approaches are equivalent~\citep{sejdinovic2013equivalence}.
Their corresponding empirical estimators all admit a U-statistics expression, and enjoy a convergence rate that is independent of the dimension.
These results can be established using tools from U-statistics theory; see, e.g.,~\citep{serfling2009approximation}.

On the other hand, Wasserstein distances provide a class of metrics on the space of probability measures with nice geometric properties~\citep{ambrosio2005gradient}.
However, it is known that its empirical estimate suffers from the curse of dimensionality~\citep{dudley1969speed,fournier2015rate,weed2019sharp,lei2020convergence}, limiting their usage in high-dimension problems.
A remedy to this issue is to introduce the entropic regularization.
\citet{genevay2019sample} showed that the plug-in estimator of the entropy regularized optimal transport cost possesses a parametric rate of convergence when the measures are compactly supported.
Their results can be extended to sub-Gaussian distributions~\citep{mena2019statistical} with tools from empirical process theory.

The independence criterion we propose uses entropy regularized optimal transport to compare the joint distribution and the product of marginals.
The empirical counterpart involves a product of two empirical measures, leading to a two-sample U-process on paired samples.
The resulting U-process requires a sophisticated analysis of its statistical behavior; common tools from empirical processes are ineffective here.
Using the decoupling technique from~\citet{pena1999decoupling} and duality theory~\citep{peyre2019computational}, we prove a rate of convergence roughly $O(\sigma^{3d} n^{-1/2})$, where $n$ is the sample size, $d$ is the ambient dimension, and $\sigma$ is the sub-Gaussian parameter, recovering previous results for two sample statistics.

\section{Entropy Regularized Optimal Transport Independence Criterion}
\label{sec:eotic}
We introduce in this section the entropy regularized optimal transport independence criterion (\teotic) and discuss its key properties.
We design an independence test based on \teotic~and develop an efficient algorithm---the Tensor Sinkhorn algorithm---to compute its test statistic. We also provide an approximation using random features for large-scale problems, as well as a differentiable program implementation for deep learning applications.

\paragraph{Notation.}
For a Euclidean space $\calZ$ equipped with the Borel $\sigma$-algebra,
let $\pspace(\calZ)$ be the set of probability measures defined on $\calZ$.
Let $(X, Y)$ be a pair of random vectors with respective dimension $d_1$ and $d_2$ following some joint distribution $\pxy \in \pspace(\reals^d)$ where $d := d_1 + d_2$.
Denote $\px \in \pspace(\reals^{d_1})$ and $\py \in \pspace(\reals^{d_2})$ the marginal distributions of $X$ and $Y$, respectively.
Given $Q \in \pspace(\reals^{d_1} \times \reals^{d_2})$ and a real-valued function $f$ on the same domain, we denote $Q[f]$ the expectation $\Expect_{(X, Y) \sim Q}[f(X, Y)]$.
We adopt the notation from the empirical process theory and write $\norm{Q}_{\calF} := \sup_{f \in \calF} \abs{Q[f]}$ for a real-valued function class $\calF$.
We say $Q$ is sub-Gaussian with parameter $\sigma^2$, denoted as subG$(\sigma^2)$, if $\Expect_Q e^{\norm{(X, Y)}^2/(2d\sigma^2)} \le 2$; see, e.g., \citet{vershynin2018high}.
We write $\Expect := \Expect_{\pxy}$ for short.

\paragraph{ETIC.}
Let $c: \reals^d \times \reals^d \rightarrow \reals_+$ be a continuous cost function satisfying $c((x, y), (x', y')) = 0$ iff $(x, y) = (x', y')$.
We introduce the entropy regularized optimal transport independence criterion (\teotic):
\begin{align}\label{eq:eotic}
  \eotic(X, Y) := \eotic_\varepsilon(X, Y) := \sdiv_\varepsilon(\pxy, \prodxy),
\end{align}
where $\sdiv_\varepsilon$ is a divergence defined as
$\sdiv_\varepsilon(\pxy, \prodxy) := S_\varepsilon(\pxy, \prodxy) - \frac12 S_\varepsilon(\pxy, \pxy) - \frac12 S_\varepsilon(\prodxy, \prodxy)$. 
Here $S_\varepsilon(\pxy, \prodxy)$ is the entropy regularized optimal transport cost between $\pxy$ and $\prodxy$, i.e.,
\begin{align}\label{eq:eot}
    \inf_{\gamma} \left[ \int c d\gamma + \varepsilon \kl{\gamma}{\pxy \otimes (\prodxy)} \right],
\end{align}
where the infimum is over $\Pi(\pxy, \prodxy)$ which is the set of couplings (or joint distributions) on $\reals^{d \times d}$ with marginals $\pxy$ and $\prodxy$, $\varepsilon > 0$ is the regularization parameter, and KL is the Kullback-Leibler divergence.
The other two terms are defined similarly and are omitted for the sake of space.

As we will show later, it is computationally convenient to work with additive cost functions, i.e., $c((x, y), (x', y')) = c_1(x, x') + c_2(y, y')$.
For this type of cost functions, we prove that the resulting \teotic~is a valid independence criterion as long as the induced \emph{Gibbs kernels}
\begin{align}\label{eq:gibbs_kernels}
    k_1(x, x') = \exp\left\{ -c_1(x, x') / \varepsilon \right\} \quad \mbox{and} \quad k_2(y, y') = \exp\left\{ -c_2(y, y') / \varepsilon \right\}
\end{align}
are positive universal; see \Cref{sec:computation} for the proof.

\begin{proposition}\label{prop:eotic_valid}
  Let $\calX \subset \reals^{d_1}$ and $\calY \subset \reals^{d_2}$ be compact subsets equipped with Lipschitz costs $c_1$ and $c_2$, respectively.
  Assume that the Gibbs kernels defined in \eqref{eq:gibbs_kernels} are positive universal.
  Then, the \teotic~is a valid dependency measure on $\pspace(\calX \times \calY)$, i.e., $T(X, Y) \ge 0$ and
  \begin{align}
      T(X, Y) = 0 \mbox{ iff } \pxy = \prodxy. 
  \end{align}
  Moreover, the claim holds true for measures with a bounded support on $\calX \times \calY = \reals^d$ with the costs $c_1(x, x') = \norm{x - x'}^p/\lambda_1$ and $c_2(y, y') = \norm{y - y'}^p/\lambda_2$ for $p \in \{1, 2\}$ and for all $\lambda_1, \lambda_2 > 0$.
\end{proposition}

A running example we consider in this paper is the weighted quadratic cost.
\begin{example}[Weighted quadratic cost]
  Let $\lambda_1, \lambda_2 \in (0, \infty)$.
  Consider the cost function
  \begin{align}
    c((x, y), (x', y')) = \frac1{\lambda_1} \norm{x - x'}^2 + \frac1{\lambda_2} \norm{y - y'}^2.
  \end{align}
  This cost induces two universal kernels $k_1(x, x') = e^{-\lVert x - x' \rVert^2 / (\varepsilon \lambda_1)}$ and $k_2(y, y') = e^{-\lVert y - y' \rVert^2 / (\varepsilon \lambda_2)}$.
  They play a similar role as the two kernels used in HSIC, and $\varepsilon\lambda_1$ and $\varepsilon \lambda_2$ serve as two kernel parameters.
\end{example}

\begin{algorithm}[t]
  \centering
  \caption{Tensor Sinkhorn Algorithm}
  \label{alg:tensor_sinkhorn}
  \begin{algorithmic}[1]
    \State {\bfseries Input:} $A$, $B$, $K_1$, and $K_2$.
    \State Initialize $U \gets \mathbf{1_{n \times n}}$ and $V \gets \mathbf{1_{n \times n}}$.
    \While{not converge}
        \State $U \gets A \oslash (K_1 V K_2^\top)$ and $V = B \oslash (K_1^\top U K_2)$.
    \EndWhile
    \State {\bfseries Output:} $U$ and $V$.
  \end{algorithmic}
\end{algorithm}

\paragraph{\teotic-Based Independence Test.}
Given an i.i.d.~sample $\{(X_i, Y_i)\}_{i=1}^n$ from $\pxy$,
we are interested in determining whether $X$ is independent of $Y$, which can be formalized as the following hypothesis testing problem:
\begin{align}
  \hnull: \pxy = \prodxy \leftrightarrow \halt: \pxy \neq \prodxy.
\end{align}
For this purpose, we use the empirical estimate of $T(X, Y)$ as the test statistic, that is,
\begin{align}
  \heotic(X, Y) := \eotic_{n, \varepsilon}(X, Y) := \sdiv_\varepsilon(\hpxy, \hprodxy),
\end{align}
where $\hpxy := \frac1n \sum_{i=1}^n \delta_{(X_i, Y_i)}$ is the empirical measure of the pairs, and $\hpx := \frac1n \sum_{i=1}^n \delta_{X_i}$ and $\hpy := \frac1n \sum_{i=1}^n \delta_{Y_i}$ are the empirical measures of the two samples, respectively.
Note that this is different from the standard plug-in estimator since the product measure $\prodxy$ is estimated by $n^2$ \emph{dependent} (rather than independent) pairs $\{(X_i, Y_j)\}_{i,j=1}^n$.
It raises challenges in the analysis of its statistical behavior as elaborated in \Cref{sec:thm}.
The statistical test (or decision rule) is then defined as
\begin{align}\label{eq:eotic_test}
    \psi(\alpha) := \ind\{T_n(X, Y) > H_n(\alpha)\},
\end{align}
where $\alpha$ is a prescribed significance level, e.g., $\alpha = 0.05$, and $H_n(\alpha)$ is a threshold chosen such that the \emph{type I error rate} $\Prob(\psi(\alpha) = 1 \mid \hnull)$ is bounded by $\alpha$.
Here $\{\psi(\alpha) = 1\}$ indicates the rejection the null hypothesis.
The \emph{(statistical) power} of the test is defined as $\Prob(\psi(\alpha) = 1 \mid \halt)$.

To avoid tuning the regularization parameter $\varepsilon$, we also consider an adaptive version of the test:
\begin{align}\label{eq:adaptive_test}
    \psi_{a}(\alpha) := \ind\left\{ \max_{\varepsilon \in \mathcal{E}} \bar T_{n, \varepsilon}(X, Y) > H_{n, \mathcal{E}}(\alpha) \right\},
\end{align}
where $\mathcal{E}$ is a finite set of positive numbers selected by the user and
\begin{align*}
  \bar T_{n, \varepsilon}(X, Y) := [T_{n, \varepsilon}(X, Y) - \Expect[T_{n, \varepsilon}(X, Y)]] / \text{Sd}(T_{n, \varepsilon}(X, Y))
\end{align*}
is the studentized version of $T_{n, \varepsilon}(X, Y)$.
In practice, $\Expect[T_{n, \varepsilon}(X, Y)]$ and $\text{Sd}(T_{n, \varepsilon}(X, Y))$ can be estimated via resampling.

\paragraph{Tensor Sinkhorn Algorithm.}
We then derive an efficient algorithm to compute the test statistic.
When $\pxy$ admits a density, $\hprodxy$ is supported on $n^2$ items $\{x_i\}_{i=1}^n \times \{y_i\}_{i=1}^n$ almost surely.
If we compute the \teotic~statistic naively using the Sinkhorn algorithm \citep{cuturi2013sinkhorn}, each iteration costs $O(n^4)$ time and space due to the matrix-vector product of sizes $n^2 \times n^2$ and $n^2 \times 1$.
To speed up its computation, we develop a variant of the Sinkhorn algorithm to solve the EOT problem between two measures supported on the Cartesian product $\{x_i\}_{i=1}^n \times \{y_i\}_{i=1}^n$.

\begin{table}[t]
\caption{Comparison of complexities, in time and in space, of Sinkhorn and Tensor Sinkhorn algorithms, Exact or Random Features approx.}
\footnotesize
\begin{center}
{\renewcommand{\arraystretch}{1.2}%
\begin{tabular}{lcc|cc}
\addlinespace[0.4em]
\toprule
& \multicolumn{2}{c}{\textbf{Sinkhorn}} & \multicolumn{2}{c}{\textbf{Tensor Sinkhorn}} \\
& \multicolumn{1}{c}{\textbf{Exact}} & \multicolumn{1}{c}{\textbf{RF}} & \multicolumn{1}{c}{\textbf{Exact}} & \multicolumn{1}{c}{\textbf{RF}} \\\midrule
\textbf{Time} & $O(n^4)$ & \tabemph{$\bm{O(n^2)}$} & $O(n^3)$ & \tabemph{$\bm{O(n^2)}$} \\
\textbf{Space} & $O(n^4)$ & $O(n^4)$ & \tabemph{$\bm{O(n^2)}$} & \tabemph{$\bm{O(n^2)}$} \\
\bottomrule
\end{tabular}
}
\label{table:time_space}
\end{center}
\end{table}

Let $A$ and $B$ be two probability measures on $\{x_i\}_{i=1}^n \times \{y_i\}_{i=1}^n$, where $x_i \in \reals^{d_1}$ and $y_j \in \reals^{d_2}$.
For convenience, both $A$ and $B$ are represented as a matrix, i.e., $A_{ij} = A(x_i, y_j)$.
For instance, if we choose $A = \hpxy$ and $B = \hprodxy$, then, in its matrix form, $A = I_{n} / n$ and $B = \mathbf{1}_{n\times n} / n^2$.
Consider an additive cost function $c$, e.g., the weighted quadratic cost, such that $c((x, y), (x', y')) = c_1(x, x') + c_2(y, y')$ for $x, x' \in \{x_i\}_{i=1}^n$ and $y, y' \in \{y_j\}_{j=1}^n$.
Let $C_1$ and $C_2$ be the cost matrices of $\{x_i\}_{i=1}^n$ and $\{y_j\}_{j=1}^n$, respectively.
Define Gibbs matrices $K_1 := e^{-C_1/\varepsilon}$ and $K_2 := e^{-C_2/\varepsilon}$, where the exponential function is element-wise.
We show in \Cref{prop:tensor_sinkhorn} in \Cref{sec:computation} that \Cref{alg:tensor_sinkhorn} can be used to compute $S_\varepsilon(A, B)$, where $\oslash$ represents element-wise division.
We refer to it as the \emph{Tensor Sinkhorn} algorithm.
Each iteration in the Tensor Sinkhorn algorithm takes $O(n^3)$ time and $O(n^2)$ space, thanks to the additive cost function being used.
This algorithm can be generalized to measures supported on the Cartesian product of $p > 2$ sets, which is also noted in \citep[Remark 4.17]{peyre2019computational}.

\paragraph{\teotic~with Random Features.}
To further speed up the computation, we apply the random feature technique introduced by~\citet{scetbon2020linear}.
On a high level, we approximate the gram matrix $K_i$ by its low-rank approximation $\xi_i \xi_i^\top$ for $i \in \{1, 2\}$, where $\xi_i \in \reals^{n \times p}$ is the matrix of random features.
Concretely, let $\rho_1$ and $\rho_2$ be two probability measures on measurable spaces $\mathcal{U}$ and $\mathcal{V}$, respectively.
Following \citet[Section 3]{scetbon2020linear}, we focus on Gibbs kernels of the form
\begin{align*}
    k_1(x, x') &= \int \varphi(x, u)^\top \varphi(x', u) d \rho_1(u) \\
    k_2(y, y') &= \int \psi(y, v)^\top \psi(y', v) d\rho_2(v),
\end{align*}
where $\varphi: \{x_i\}_{i=1}^n \times \mathcal{U} \rightarrow \mathbb{R}_+^{q_1}$ and $\psi: \{y_i\}_{i=1}^n \times \mathcal{V} \rightarrow \mathbb{R}_+^{q_2}$.
Note that the Gibbs kernels induced by the weighted quadratic cost admit this expression.
For $p \in \mathbb{N}_+$, we obtain two i.i.d.~samples $\{u_i\}_{i=1}^p$ and $\{v_i\}_{i=1}^p$ from $\rho_1$ and $\rho_2$, respectively.
We denote $\bm{u} := (u_1, \dots, u_p)$ and approximate $k_1(x, x')$ by
\begin{align*}
    k_{1,\bm{u}}(x, x') := \frac1p \sum_{k=1}^p \varphi(x, u_k)^\top \varphi(x', u_k).
\end{align*}
This new kernel induces the cost $c_{1, \bm{u}}(x, x') := -\varepsilon \log{k_{1,\bm{u}}(x, x')}$.
Similarly, we define $\bm{v}$, $k_{2, \bm{v}}(y, y')$, and $c_{2, \bm{v}}(y, y')$.
It is clear that \Cref{alg:tensor_sinkhorn} with inputs $A$, $B$, $K_{1, \bm{u}}$, and $K_{2, \bm{v}}$ solves the entropy-regularized optimal transport problem with cost $c_{\bm{u}, \bm{v}}((x, y), (x', y')) := c_{1, \bm{u}}(x, x') + c_{2, \bm{v}}(y, y')$.
Let $S_{\varepsilon, c_{\bm{u}, \bm{v}}}(A, B)$ be the entropic cost.
The next proposition provides a high-probability guarantee for this random feature approximation.
\begin{assumption}\label{asmp:random_feature}
    There exists a constant $C > 0$ such that, for all $x, x' \in \{x_i\}_{i=1}^n$, $y, y' \in \{y_j\}_{j=1}^n$, $u \in \mathcal{U}$, and $v \in \mathcal{V}$,
    \begin{align*}
        \varphi(x, u)^\top \varphi(x', u) / k_1(x, x') &\le C \\
        \psi(y, v)^\top \psi(y', v) / k_2(y, y') &\le C.
    \end{align*}
\end{assumption}
\begin{proposition}\label{prop:rf_tensor_sinkhorn}
    Let $\delta > 0$, $\tau > 0$, and $p = \Omega\left( \frac{C^2}{\tau^2} \log{\frac{n}{\delta}} \right)$.
    Under \Cref{asmp:random_feature}, with probability at least $1 - \delta$, it holds that
    \begin{align*}
        \abs{S_{\varepsilon, c_{\bm{u}, \bm{v}}}(A, B) - S_{\varepsilon, c}(A, B)} \le \tau.
    \end{align*}
\end{proposition}

Replacing $K_1$ and $K_2$ by their random feature approximations $K_{1, \bm{u}}$ and $K_{2, \bm{v}}$ in \Cref{alg:tensor_sinkhorn} leads to an algorithm with $O(pn^2)$ time complexity and $O(n^2)$ space complexity in each iteration.
We note that if one applies the random feature approximation directly to the original Sinkhorn algorithm, then the resulting algorithm would have the same time complexity $O(pn^2)$ but $O(n^4)$ space complexity; see \Cref{table:time_space} for a comparison.

\paragraph{Dual Representation.}
The entropy regularized formulation of OT \eqref{eq:eot} is known as the \emph{Schr\"odinger bridge problem} \citep{follmer88,leonard12,leonardsurvey} in continuum and the \emph{Sinkhorn distance} \citep{cuturi2013sinkhorn,ferradans2014regularized} in the discrete case.
It admits a dual representation \citep{genevay2016stochastic}:
\begin{align*}
  \sup_{f,g \in \calC(\reals^{d_1} \times \reals^{d_2})} \bigg[ \int f d\pxy + \int g d(\prodxy) + \varepsilon - \varepsilon \int e^{-D_\varepsilon(x, y, x', y')} d\pxy(x, y) d\px(x') d\py(y') \bigg],
\end{align*}
where $\calC$ is the set of real-valued continuous functions and $D_\varepsilon(x, y, x', y') = \frac1\varepsilon[c_1(x, x') + c_2(y, y') - f(x, y) - g(x', y')]$.
Due to \citep{csiszar75,ruschendorf93}, the optimal potentials $(f_\varepsilon, g_\varepsilon)$, to be called the \emph{Schr\"odinger potentials}, satisfy
\begin{equation}\label{eq:optim_potentials}
  \begin{split}
    \int e^{-D_\varepsilon(x, y, x', y')} d\px(x') d\py(y') &\txtover{a.s.}{=} 1 \\
    \int e^{-D_\varepsilon(x, y, x', y')} d\pxy(x, y) &\txtover{a.s.}{=} 1.
  \end{split}
\end{equation}

\paragraph{Gradient Backpropagation through~\teotic.}
We describe here how \teotic~can fit into a differentiable programming framework, i.e., how one can run the reverse mode automatic differentiation through statistical quantities based on~\teotic. Recently, \citet{li2021selfsupervised} proposed a self-supervised learning approach using HSIC which we summarize below.
Let $(W, Y)$ be a pair of image and its identity.
Given an i.i.d.~sample $\{(W_i, Y_i)\}_{i=1}^n$, the goal is to learn a feature embedding model $\phi_\theta$ such that the dependence between the image feature $X := \phi_\theta(W)$ and its identity $Y$ is maximized, i.e., $\max_{\theta \in \Theta} \hsic_n(\phi_\theta(W), Y)$.
Similarly, one could also maximize the dependence measured by \teotic~instead. This boils down to gradient backpropagating through $T_n(\phi_\theta(W), Y)$.
We use the strategy in \citep[Section 9.1.3]{peyre2019computational} and illustrate it on the entropy regularized OT $S_\varepsilon(\hpxy, \hprodxy)$ in \eqref{eq:eot}.
For the forward pass, we construct the computational graph via the following steps.
Firstly, we run \Cref{alg:tensor_sinkhorn} (or its random feature variant) with $A = I_n/n$, $B = \mathbf{1}_{n\times n} / n^2$, $K_1 = \big(k_1(\phi_\theta(W_i), \phi_\theta(W_j))\big)_{n\times n}$, and $K_2 = \big(k_2(Y_i, Y_j)\big)_{n\times n}$ for $L$ iterations to get $U^{(L)}$ and $V^{(L)}$.
Secondly, we obtain the associated Schr\"odinger potentials $F^{(L)} := \varepsilon \log{U^{(L)}}$ and $G^{(L)} := \varepsilon \log{V^{(L)}}$.
Thirdly, we approximate $S_\varepsilon(\hpxy, \hprodxy)$ by $\hat S_\varepsilon(\theta) := \ip{F^{(L)}, A}_{\mathbf{F}} + \ip{G^{(L)}, B}_{\mathbf{F}}$ where $\ip{\cdot, \cdot}_{\mathbf{F}}$ is the Frobenius inner product.
For the backward pass, we call the reverse mode automatic differentiation to evaluate $\nabla_{\theta} \hat S_\varepsilon(\theta)$.
Since computing $\hat S_\varepsilon(\theta)$ only requires simple operations between matrices, the time complexity of the above procedure is of the same order as the one of \Cref{alg:tensor_sinkhorn} for the computation of $\hat S_\varepsilon(\theta)$.

\paragraph{Connection to Previous Work.}
As shown in \Cref{sec:computation}, $T(X, Y)$ tends to the OT-based independence criterion $\ot(\pxy, \prodxy)$ as $\varepsilon \rightarrow 0$.
If the cost $c$ is chosen as the Euclidean distance to the power $p \ge 1$, it induces a distance (known as the Wasserstein-$p$ distance) on the space of probability measures \citep{villani2016topics}.
As a result, $\ot(\pxy, \prodxy)$ is a valid independence criterion, i.e., $\ot(\pxy, \prodxy) = 0$ iff $\pxy = \prodxy$.
The study of this independence criterion can be dated back to Gini; see \citep{cifarelli2017gini} for a discussion.
Its normalized version---the so-called Wasserstein correlation coefficient---has recently gained attention in \citep{wiesel2021measuring,mordant2021measuring,nies2021transport}.
When $\varepsilon \rightarrow \infty$, $T(X, Y)$ tends to 0 if the cost is additive;
if the cost is multiplicative, i.e., $c((x, y), (x', y')) = c_1(x, x') c_2(y, y')$, it recovers the negative of HSIC with kernels $c_1$ and $c_2$.

The quantity $\bar S_\varepsilon$ is known as the \emph{Sinkhorn divergence}.
It has been used in two-sample problems, where the goal is to quantify the distance of two distributions given i.i.d.~samples from each of them.
In particular, it is applied to two-sample testing \citep{ramadas2017wasserstein} and generative modeling \citep{genevay2018learning}.
It is shown in \citep{feydy2019interpolating} that $\sdiv_\varepsilon$ defines a semi-metric (metric without the triangle inequality) on the space of probability measures with bounded support if the Gibbs kernel induced by the cost is positive universal.
The limiting behavior of the empirical estimator is to date not known in the literature, though non-asymptotic bounds are attainable using results in \citep{genevay2019sample,mena2019statistical}.
Our results also recover the two-sample case.

\section{Main Results}
\label{sec:thm}
We give non-asymptotic bounds for the \teotic~statistic with quadratic cost.
We present the main results and their proof sketches here.
We use $C$ to denote a constant whose value may change from line to line, where subscripts are used to emphasize the dependency on other quantities.
For instance, $C_d$ represents a constant depending only on the dimension $d$.
The detailed proofs are deferred to Appendices \ref{sec:consistency} and \ref{sec:tail_bounds}.

\paragraph{Consistency.}
We first show that the \teotic~statistic is a consistent estimator of its population counterpart under both the null and alternative.
\begin{assumption}\label{asmp:quadratic_subg}
  We make the following assumptions:
  \begin{enumerate}[label=(\roman*)]
      \item $c$ is chosen as the quadratic cost.
      \item $\px$ and $\py$ are subG$(\sigma^2)$.
  \end{enumerate}
\end{assumption}

The quadratic cost is chosen for the sake of concision.
We extend the results to weighted quadratic cost in Appendix \ref{sec:consistency}.

\begin{theorem}\label{thm:consistency}
  Under \Cref{asmp:quadratic_subg}, we have
  \begin{align*}
    \Expect\abs{\heotic(X, Y) - \eotic(X, Y)} \le C_d \left(1 + \frac{\sigma^{\lceil 5d/2 \rceil + 6}}{\varepsilon^{\lceil 5d/4 \rceil + 3}} \right)\frac{\varepsilon}{\sqrt{n}}.
  \end{align*}
\end{theorem}
\begin{remark}
  According to \Cref{thm:consistency}, when $\varepsilon = \varepsilon_n$ is chosen such that $\varepsilon_n = \omega(n^{-1/(\lceil 5d/2 \rceil + 4)})$ and $\varepsilon_n = o(1)$, we have $T_n(X, Y)$ converges in $\lone$ to $\ot(\pxy, \prodxy)$ as $n \rightarrow \infty$.
\end{remark}

We can upper bound the above $\lone$ loss by the supremum of an empirical process and a U-process
\begin{align*}
    \norm{\hpxy - \pxy}_{\calF^s}^2 \quad \mbox{and} \quad \norm{\hprodxy - \prodxy}_{\calF^s}^2,
\end{align*}
respectively,
where $\calF^s$ is the set of real-valued functions satisfying
\begin{equation*}
    \begin{split}
        \abs{f(x, y)} &\le C_{s,d} (1 + \norm{(x, y)}^2) \\
        \abs{D^\alpha f(x, y)} &\le C_{s,d} (1 + \norm{(x, y)}^{\abs{\alpha}}), \quad \forall 1 < \abs{\alpha} \le s.
    \end{split}
\end{equation*}
\citet{mena2019statistical} used a similar strategy in their proofs.
Empirical process theory has a long history in statistics and there are well-established tools to control them; see, e.g., \citep{vaart1996weak}.
However, the theory of U-processes is much less well-developed.
Moreover, many of the previous works focus on one-sample U-processes; see, e.g., \citep{pena1999decoupling}. The second U-process here is a two-sample U-process on a paired sample, bringing about additional challenges in its analysis, compared to, e.g., \citet{mena2019statistical}.
In order to control it, we develop the following results.

The first result is a metric entropy bound for \emph{degenerate two-sample U-processes}.
The main challenge comes from the dependence among the summands in $\sum_{i,j=1}^n f(X_i, Y_j)$.
We get around that using the decoupling technique presented in \citep{pena1999decoupling}.
\begin{proposition}\label{prop:prod_emp_bound}
  Let $\calF$ be a class of real-valued functions that are degenerate under $\prodxy$, i.e.,
  \begin{align*}
      \Expect_{\prodxy}[f(X, Y) \mid X] \txtover{a.s.}{=} \Expect_{\prodxy}[f(X, Y) \mid Y] \txtover{a.s.}{=} 0
  \end{align*}
  for any $f \in \calF$.
  Under \Cref{asmp:quadratic_subg}, we have
  \begin{align*}
      \Expect\norm{\hprodxy - \prodxy}_{\calF}^2
      \le \frac{C}n \Expect\left( \int_0^{B} \sqrt{\log{N(\tau, \calF, \ltwo(\hprodxy))}} d\tau \right)^2,
  \end{align*}
  where $B$ is any measurable upper bound of $2\max_{f \in \calF} \norm{f}_{\ltwo(\hprodxy)}$.
\end{proposition}

\begin{remark}
  In classical two-sample U-statistics literature, it is usually assumed that the two samples are independent, i.e., $X$ is independent of $Y$.
  However, \Cref{prop:prod_emp_bound} allows the sample to be paired since $(X, Y) \sim \pxy$.
\end{remark}

With \Cref{prop:prod_emp_bound} at hand, we can control the U-process $\lVert \hprodxy - \prodxy \rVert_{\calF^s}^2$ by upper bounding its covering number $N(\tau, \calF^s, \ltwo(\hprodxy))$.
The proof is inspired by~\citep{mena2019statistical} and relies on a result in \citep[Chapter 2.7]{vaart1996weak} to control the covering number of a class of smooth functions.
\begin{proposition}\label{prop:cover_number_prod}
  Under \Cref{asmp:quadratic_subg},
  there exists a random variable $L \ge 1$ depending on the samples $\{(X_i, Y_i)\}_{i=1}^n$ with $\Expect[L] \le 2$ such that, for any $s \ge 2$,
  \begin{align*}
    \log{N(\tau, \calF^s, \ltwo(\hprodxy))} &\le C_{s,d} \tau^{-d/s} L^{d/2s} (1 + \sigma^{2d})
  \end{align*}
  and
  \begin{align*}
      \max_{f \in \calF^s} \norm{f}^2_{\ltwo(\hprodxy)} &\le C_{s,d}(1 + L \sigma^4).
  \end{align*}
  In particular, when $s > d/2$, we have
  \begin{align*}
    \Expect\norm{\hprodxy - \prodxy}_{\calF^s}^2 \le C_{s,d} (1 + \sigma^{2d+4}) \frac1n.
  \end{align*}
\end{proposition}

\paragraph{Exponential Tail Bound.}
We also prove an exponential tail bound for the \teotic~statistic.
It follows from \Cref{thm:consistency} and the McDiarmid inequality.

\begin{theorem}\label{thm:tail_bound}
  Let $c$ be the quadratic cost.
  Assume that $\px$ and $\py$ are supported on a bounded domain of radius $D$.
  Then we have, with probability at least $1 - \delta$,
  \begin{align*}
      \abs{\heotic(X, Y) - \eotic(X, Y)}
      \le C_d \left( 1 + \frac{D^{5d+16}}{\varepsilon^{5d/2 + 8}} \sqrt{\log{\frac6\delta}} \right) \frac{\varepsilon}{\sqrt{n}}.
  \end{align*}
\end{theorem}

Under $\hnull$, we have $T(X, Y) = 0$, so \Cref{thm:tail_bound} implies that
\begin{align*}
    \abs{\heotic(X, Y)} > C_d \left( 1 + \frac{D^{5d+16}}{\varepsilon^{5d/2 + 8}} \sqrt{\log{\frac6\delta}} \right) \frac{\varepsilon}{\sqrt{n}}
\end{align*}
with probability at most $\delta$.
It gives an estimate of the tail behavior of $\heotic(X, Y)$ which suggests that the critical value $H_n(\alpha)$ in \eqref{eq:eotic_test} should be of order $O(n^{-1/2})$.
Under $\halt$, \Cref{thm:tail_bound} implies that
\begin{align*}
    \heotic(X, Y) &> \eotic(X, Y) - C_d \left( 1 + \frac{D^{5d+16}}{\varepsilon^{5d/2 + 8}} \sqrt{\log{\frac6\delta}} \right) \frac{\varepsilon}{\sqrt{n}}
\end{align*}
with probability at least $1 - \delta$.
When $\eotic(X, Y) > 0$, it is clear that the right hand side in the above inequality exceeds the threshold $H_n(\alpha)$ for large $n$.
Hence, the \teotic~test has power converging to 1 as $n \rightarrow \infty$.

\section{Experiments}
\label{sec:experiments}
\begin{figure*}[t]
    \centering
    \includegraphics[width=0.7\textwidth]{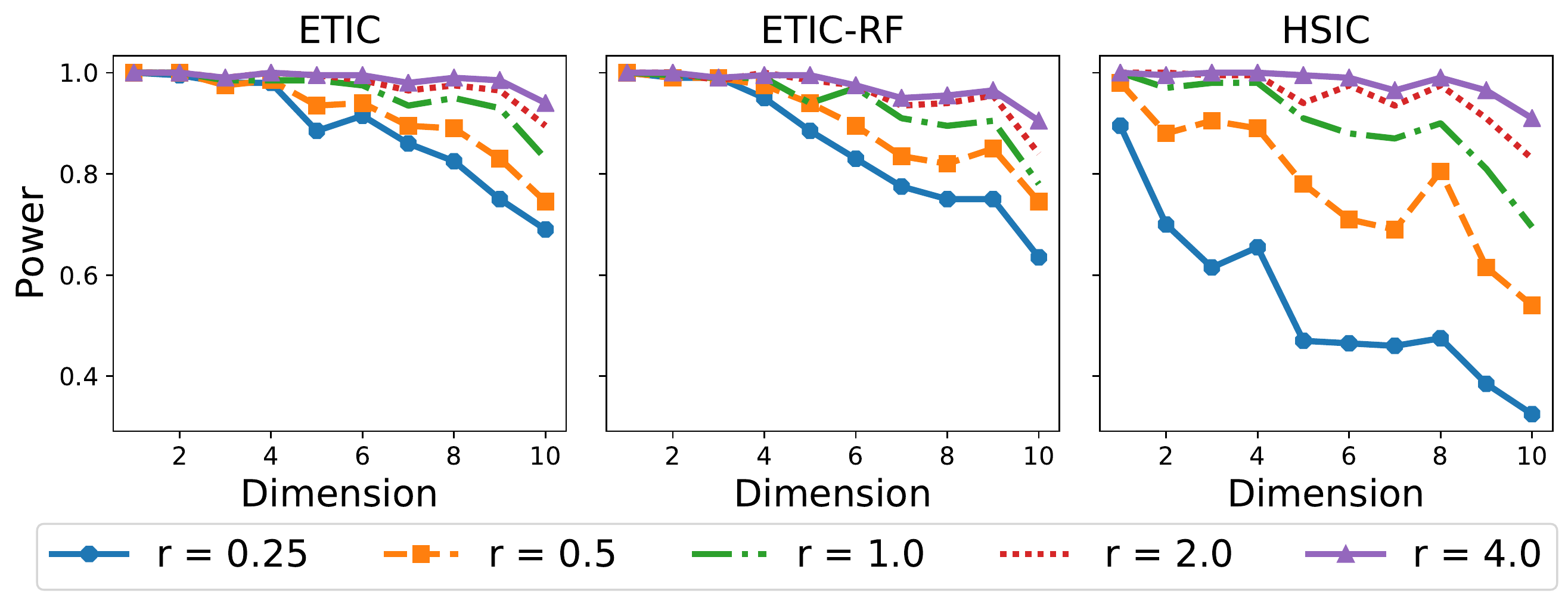}
    \vspace{.1in}
    \caption{Power versus dimension in the linear dependency model \eqref{eq:linear_dependency}.}
    \label{fig:linear_dependency}
\end{figure*}

We examine the empirical behavior of the proposed \teotic~test for independence testing on both synthetic and real data.
We consider synthetic benchmarks from~\citep{gretton2007kernel,jikrettum2017adaptive,zhang2018large} and revisit an application from~\citep{gretton2007kernel} with recent feature representations for text data.
The performance of the adaptive \teotic~test is also investigated but is deferred to \Cref{sub:adaptive_etic} due to the space limit.
The code to reproduce the experiments is available online\footnote{\url{https://github.com/langliu95/etic-experiments}.}.

\begin{figure*}[t]
    \centering
    \includegraphics[width=0.7\textwidth]{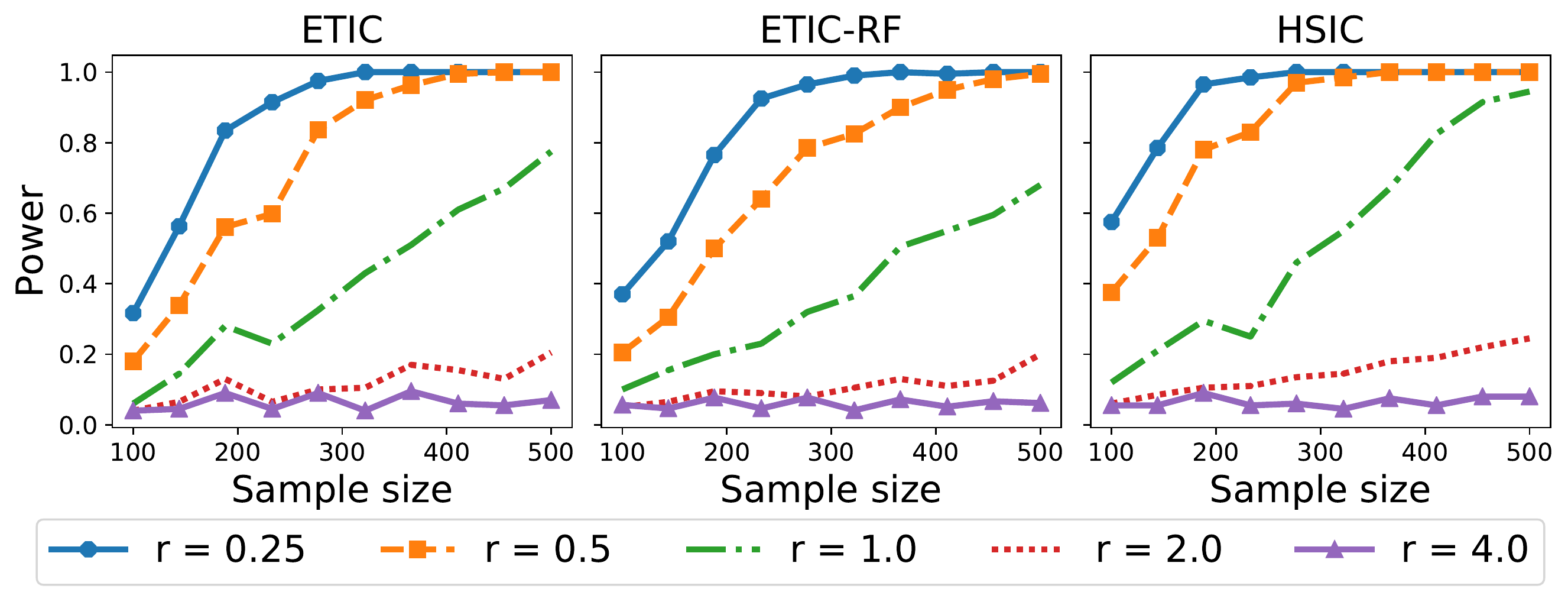}
    \vspace{.1in}
    \caption{Power versus sample size in the Gaussian-sign model \eqref{eq:gaussian_sign}.}
    \label{fig:gaussian_sign}
\end{figure*}

We focus on the weighted quadratic cost
\begin{align*}
    c((x, y), (x', y')) = \frac1{\lambda_1} \norm{x - x'}^2 + \frac1{\lambda_2} \norm{y - y'}^2.
\end{align*}
For convenience, we absorb the regularization parameter $\varepsilon$ into the weights and set $\varepsilon = 1$.
It then induces two Gibbs kernels
\begin{align*}
    k_1(x, x') = \exp\left\{ -\norm{x - x'}^2 / \lambda_1 \right\} \quad \text{and} \quad k_2(y, y') = \exp\left\{ -\norm{y - y'}^2 / \lambda_2 \right\}
\end{align*}
with $\lambda_i$ being the parameter of kernel $k_i$ for $i \in \{1, 2\}$.
To select the weights, we apply the median heuristic \citep{gretton2007kernel} widely used for HSIC, i.e.,
\begin{align*}
    \lambda_1 = r_1 M_x \quad \mbox{and} \quad \lambda_2 = r_2 M_y
\end{align*}
with $r_1$ and $r_2$ ranging from 0.25 to 4, where $M_x$ and $M_y$ are the medians of the quadratic costs $\{\norm{X_i - X_j}^2\}$ and $\{\norm{Y_i - Y_j}^2\}$, respectively.
We also examine its variant \eoticrf~discussed in \Cref{sec:eotic}, where the number of random features is set to be 100 unless otherwise noted.
We compare them with the HSIC statistic with kernels $k_1$ and $k_2$.
For a fair comparison, we calibrate these tests by a Monte Carlo resampling technique \citep{feuerverger1993consistent} with 200 permutations.                            
For each of the experiment, we repeat the whole procedure 200 times and report the rejection frequency as either the type I error rate (when the null is true) or power (when the null is not true).
Note that, even though we are using the same $\lambda_1$ and $\lambda_2$ in the cost and kernels, that does \emph{not} mean we should compare \teotic~and HSIC under the same hyperparameters.
Our goal is to explore their performance over a range of values of the hyperparameters controlling the regularization penalties.

Our main findings are: 1) Both \teotic~and \eoticrf~are consistent in power as the sample size approaches infinity. 2) In some scenarios, \teotic~and \eoticrf~outperforms HSIC significantly; in the linear dependency model in particular, their power is much more robust than HSIC to the value of the hyperparameters. 3) \eoticrf~performs reasonably good compared to \teotic~with a moderate number (i.e., 100) of random features. 4) All three tests benefit from large hyperparameters in detecting simple linear dependency, but smaller values lead to higher power when the dependency is more complicated.

\paragraph{Hilbert-Schmidt Independence Criterion.}
Before we present our results, let us recall the definition of HSIC.
Let $k: \reals^{d_1} \times \reals^{d_1} \rightarrow \reals$ and $l: \reals^{d_2} \times \reals^{d_2} \rightarrow \reals$ be two positive semi-definite kernels.
The Hilbert-Schmidt independence criterion (HSIC) between $X$ and $Y$, $\hsic(X, Y)$, is defined as
\begin{align*}
    \Expect[k(X, X') l(Y, Y')] + \Expect[k(X, X')] \Expect[l(Y, Y')] - 2\Expect[\Expect[k(X, X') \mid X] \Expect[l(Y, Y') \mid Y]],
\end{align*}
where $(X', Y')$ is an independent copy of $(X, Y)$.
Given an i.i.d.~sample $\{(X_i, Y_i)\}_{i=1}^n$ from $\pxy$, we can estimate $\hsic(X, Y)$ by
\begin{align*}
    \frac1{n^2} \sum_{i,j=1}^n k_{ij} l_{ij} + \frac1{n^4} \sum_{i,j,s,t=1}^n k_{ij} l_{st} - \frac2{n^3} \sum_{i,j,s=1}^n k_{ij} l_{is},
\end{align*}
where $k_{ij} := k(X_i, X_j)$ and $l_{ij} := l(Y_i, Y_j)$.
We refer to it as the HSIC statistic.

\subsection{Synthetic Data}
\label{sub:synthetic}

\begin{figure*}[t]
    \centering
    \includegraphics[width=0.7\textwidth]{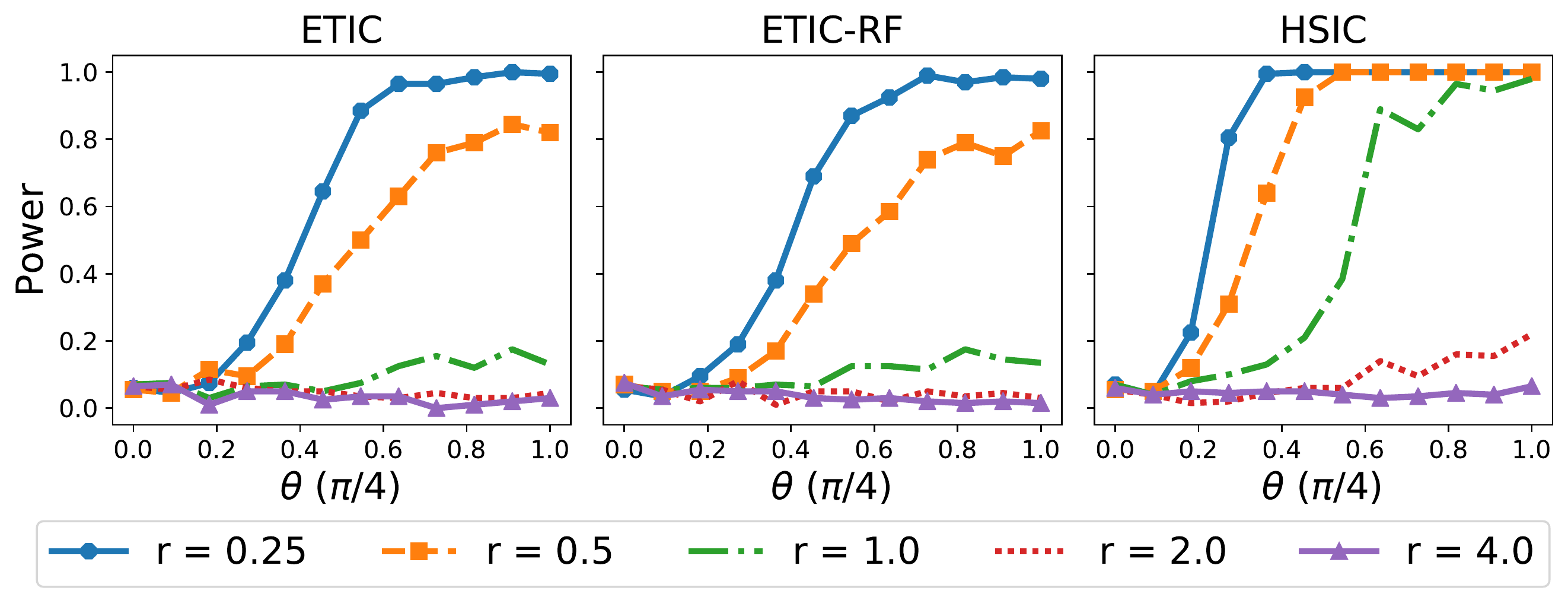}
    \vspace{.1in}
    \caption{Power versus parameter in the subspace dependency model.}
    \label{fig:linear_mixture}
\end{figure*}

We first compare the performance of \teotic~and \eoticrf~with HSIC on synthetic data.
We consider synthetic benchmarks from~\citep{zhang2018large}, \citep{jikrettum2017adaptive}, and~\citep{gretton2007kernel}.
To facilitate the exploration, we set $r_1 = r_2 = r \in \{0.25, 0.5, 1, 2, 4\}$ in this section.
Moreover, we examine the performance of two other independence tests discussed in \citep{gretton2008nonparametric} and summarize the results and findings in \Cref{sub:comparison}.

\paragraph{Linear Dependency.}
We begin with a simple linear dependency model.
Concretely,
\begin{align}\label{eq:linear_dependency}
    X \sim \calN_{d}(0, I_{d}) \quad \mbox{and} \quad Y = X_1 + Z,
\end{align}
where $X_1$ is the first coordinate of $X$, and $Z \sim \calN(0, 1)$ is independent with $X$.
We fix $n = 50$ and plot the power versus $d \in [1, 10]$ in \Cref{fig:linear_dependency}.
All the tests have decaying power as the dimension increases.
This is as expected since larger dimension results in weaker dependency between $X$ and $Y$.
It is clear that the power of both \teotic~and HSIC increases as $r$ \emph{increases}, with the former more robust than the latter.
While the performance of HSIC is similar to \teotic~when $r$ is large, it is much worse than \teotic~when $r$ is small.
As for \eoticrf, it has similar power curves as \teotic.

\paragraph{Gaussian Sign.}
We then consider a Gaussian sign model, i.e.,
\begin{align}\label{eq:gaussian_sign}
    X \sim \calN_{d}(0, I_{d}) \quad \mbox{and} \quad Y = \abs{Z} \prod_{i=1}^{d} \text{sgn}(X_i),
\end{align}
where $\text{sgn}(\cdot)$ is the sign function and $Z \sim \calN(0, 1)$ is independent with $X$.
This problem is challenging since $Y$ is independent with any strict subset of $\{X_1, \dots, X_{d}\}$.
We fix $d = 3$ and plot the power versus $n \in [100, 500]$ in \Cref{fig:gaussian_sign}.
All the tests have improved power as the sample size increases.
Additionally, they all benefit from a \emph{small} regularization parameter, with HSIC performs the best and the other two perform similarly in this particular example.

\begin{figure*}[t]
\centering
\includegraphics[width=0.6\textwidth]{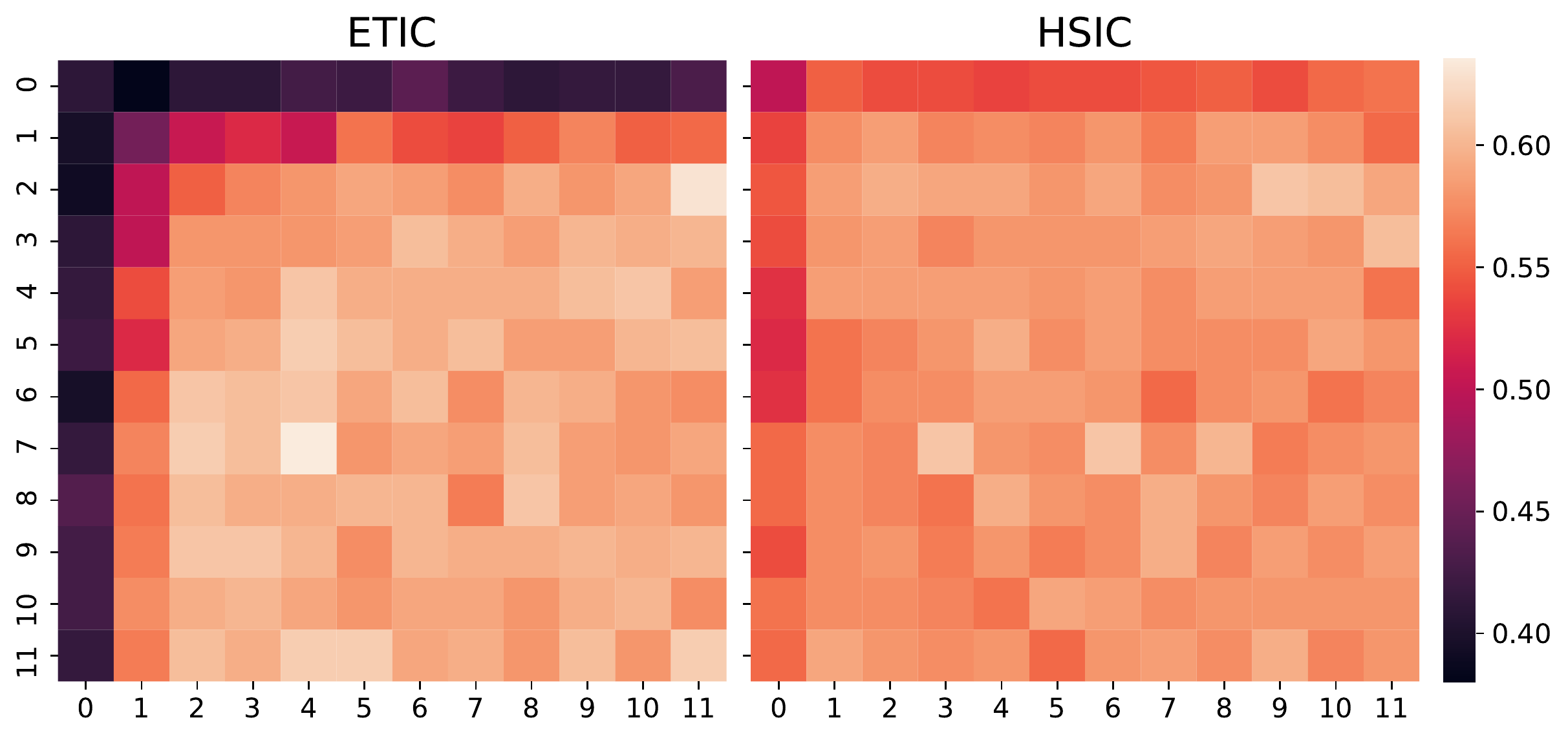}
\vspace{.1in}
\caption{Heatmaps of power on the partially dependent sample of the bilingual data. The $x$-axis is for $r_1$ and $y$-axis is for $r_2$. The indices from 0 to 11 correspond to equally spaced values from 0.25 to 4. Lighter color indicates larger power.}
\label{fig:bilingual}
\end{figure*}

\paragraph{Subspace Dependency.}
One important application of independence testing is independent component analysis \citep{gretton2005kernel}, which involves separating random variables from their linear mixtures.
We construct our data by i) generating $n$ i.i.d.~copies of two random variables following independently $0.5 \calN(0.98, 0.04) + 0.5 \calN(-0.98, 0.04)$, ii) mixing the two random variables by a rotation matrix parameterized by $\theta \in [0, \pi/4]$ (larger $\theta$ leads to stronger dependency), iii) appending $\calN_{d-1}(0, I_{d-1})$ to each of the two mixtures, and iv) multiplying each vector by an independent random $d$-dimensional orthogonal matrix.
We refer to it as the \emph{subspace dependency model}.
We fix $n = 64$, $d = 2$, and plot the power versus $\theta \in [0, \pi/4]$ in \Cref{fig:linear_mixture}.
As expected, the power of all three tests improves as $\theta$ becomes closer to $\pi / 4$.
Moreover, they all have improved power as $r$ \emph{decreases}.
\teotic~and \eoticrf~performs similarly, and they are outperformed by HSIC in this particular example.

\subsection{Dependency between Bilingual Text}
\label{sub:real}

Inspired by \citet{gretton2007kernel}, we now investigate the performance of the proposed tests on bilingual data using recent developments in natural language processing.
Our dataset is taken from the parallel European Parliament corpus \citep{koehn2005europarl} which consists of a large number of documents of the same content in different languages.
Note that it is also used in \citep{bounliphone2015low} to test for relative dependency.
For the hyperparameters, we consider different values of $r_1$ and $r_2$ ranging from $0.25$ to $4$.

To be more specific, we randomly select $n = 64$ English documents and a paragraph in each document from the corpus.
We then 1) pair each paragraph with the corresponding paragraph in French to form the dependent sample, 2) pair each paragraph with a random paragraph in the same document in French to form the partially dependent sample, and 3) pair each paragraph with a random paragraph in French to form the independent sample.

Finally, we use LaBSE \citep{feng2020languageagnostic} to embed all the paragraphs into a common feature embedding space of dimension 768 and perform independence testing on these feature vectors. LaBSE is a state-of-the-art, language agnostic, sentence embedding model based on Bidirectional Encoder Representations from Transformers (BERT).  
This allows us to revisit the idea of~\citet{gretton2007kernel} yet with more modern feature embeddings.

Both \teotic~and HSIC perform perfectly on the dependent sample (with power 1) and the independent sample (with low type I error) across all values of $r_1$ and $r_2$ considered.
The results on the partially dependent sample is shown in \Cref{fig:bilingual}.
\teotic~performs better than HSIC when one of $r_1$ and $r_2$ is large; while HSIC has larger power when $r_1$ or $r_2$ is small. Overall \teotic~appears to perform better than HSIC for large amounts of regularization parameters.

As for \eoticrf, the high-dimensional natural of the feature embeddings imposes challenges on the random feature approximation.
For its performance to be comparable, we first apply the principal component analysis (PCA) on the text embeddings to reduce the dimension, and then we perform \eoticrf~on the low-dimensional features.
As presented in \Cref{sub:etic_rf_text}, the random feature approximation equipped with PCA demonstrates similar performance as the exact \teotic~with enough random features.

\paragraph{Conclusion.}
We introduced a new independence criterion \teotic~based on entropy regularized optimal transport. The proposed criterion can be approximated using a random feature approximation. We established non-asymptotic bounds using U-process theory and optimal transport theory. The experimental results show that \teotic~can exhibit stable behavior w.r.t.~its hyperparameters.
The extension of \teotic~to multi-way dependence is an interesting venue for future work.

\subsubsection*{Acknowledgements}
The authors would like to thank M. Scetbon for fruitful discussions.
L. Liu is supported by NSF CCF-2019844 and NSF DMS-2023166.
S. Pal is supported by NSF DMS-2052239 and a PIMS CRG (PIHOT). Z. Harchaoui is supported by NSF CCF-2019844, NSF DMS-2134012, NSF DMS-2023166, CIFAR-LMB, and faculty research awards.
Part of this work was done while Z. Harchaoui was visiting the Simons Institute for the Theory of Computing.
The authors would like to thank the Kantorovich Initiative of the Pacific Institute for the Mathematical Sciences (PIMS) for supporting this collaboration.

\bibliographystyle{abbrvnat}
\bibliography{biblio}

\clearpage
\onecolumn
\appendix

%%%%%%%%%%%%%%%%%%%%%%%%%%%%%%%%%%%%%%%%%%%%%%%%%%%%%%%%%%%%%%%%%%%%%%%%
%%%%%%%%%%%%%%%%%%%%%%%%%%%%%%%%%%%%%%%%%%%%%%%%%%%%%%%%%%%%%%%%%%%%%%%%
\section{Properties of \teotic}
\label{sec:computation}
%%%%%%%%%%%%%%%%%%%%%%%%%%%%%%%%%%%%%%%%%%%%%%%%%%%%%%%%%%%%%%%%%%%%%%%%
%%%%%%%%%%%%%%%%%%%%%%%%%%%%%%%%%%%%%%%%%%%%%%%%%%%%%%%%%%%%%%%%%%%%%%%%

In this section, we prove the properties of \teotic~discussed in \Cref{sec:eotic}.
For the sake of generality, we state the problem for general notations $P$ and $Q$ while keeping in mind that $P, Q \in \{\pxy, \prodxy\}$ in our case.
Let $P \in \pspace(\reals^{d_1} \times \reals^{d_2})$ and $P_X$ and $P_Y$ be the marginals on $\reals^{d_1}$ and $\reals^{d_2}$, respectively.
Define $Q$, $Q_X$, and $Q_Y$ similarly.
We are interested in the EOT cost between $P$ and $Q$ under the cost function $c$:
\begin{align}\label{eq:a_sinkhorn_div}
    S_\varepsilon(P, Q) := \inf_{\gamma \in \Pi(P, Q)} \left[ \int c d\gamma + \varepsilon \kl{\gamma}{P \otimes Q} \right].
\end{align}
When $\varepsilon = 0$, $S_0(P, Q)$ is the optimal transport cost between $P$ and $Q$.
When $\varepsilon > 0$, it admits a dual representation:
\begin{align}\label{eq:a_dual}
    S_\varepsilon(P, Q) := \sup_{f, g \in \mathcal{C}(\reals^{d_1} \times \reals^{d_2})} \left[ \int f dP + \int g dQ + \varepsilon - \varepsilon \int e^{\frac1\varepsilon [f(z) + g(z') - c(z, z')]} dP(z) dQ(z') \right].
\end{align}
The Schr\"odinger bridge potentials $(f_\varepsilon, g_\varepsilon)$ satisfy the optimality conditions:
\begin{equation}\label{eq:a_optim_potentials}
    \begin{split}
        \int e^{\frac1\varepsilon [f_\varepsilon(z) + g_\varepsilon(z) - c(z, z')]} dQ(z') &\txtover{a.s.}{=} 1 \\
        \int e^{\frac1\varepsilon [f_\varepsilon(z) + g_\varepsilon(z) - c(z, z')]} dP(z) &\txtover{a.s.}{=} 1.
    \end{split}
\end{equation}

We first derive the limit \teotic~as $\varepsilon \rightarrow 0$ and $\varepsilon \rightarrow \infty$.
\begin{proposition}\label{prop:etic_eps_limit}
    Let $c$ be a continuous cost function.
    If either $c$ is bounded or $P$ and $Q$ have compact support,
    it holds that
    \begin{align}\label{eq:etic_eps_infty}
        \eotic_\varepsilon(X, Y) \rightarrow
        \begin{cases}
            0 & \mbox{if } c = c_1 \oplus c_2 \\
            -\frac12 \hsic_{c_1, c_2}(X, Y) & \mbox{if } c = c_1 \otimes c_2,
        \end{cases}
        \quad \mbox{as } \varepsilon \rightarrow \infty.
    \end{align}
    Moreover, if both $P$ and $Q$ are densities (or discrete measures), then
    \begin{align}\label{eq:etic_eps_zero}
        \eotic_\varepsilon(X, Y) \rightarrow S_0(\pxy, \prodxy), \quad \mbox{as } \varepsilon \rightarrow 0.
    \end{align}
\end{proposition}
\begin{proof}
    To show \eqref{eq:etic_eps_infty}, we claim that, for all $P, Q \in \pspace(\reals^d)$,
    \begin{align}\label{eq:eot_sandwich}
        S_0(P, Q) \le S_\varepsilon(P, Q) \le (P \otimes Q)[c],
    \end{align}
    and
    \begin{align}\label{eq:a_eps_infinity}
        \lim_{\varepsilon \rightarrow \infty} S_\varepsilon(P, Q) = (P \otimes Q)[c].
    \end{align}
    In fact, for any $\varepsilon_1 < \varepsilon_2$, we have
    \begin{align*}
        \int cd\gamma + \varepsilon_1 \kl{\gamma}{P \otimes Q} \le \int cd\gamma + \varepsilon_2 \kl{\gamma}{P \otimes Q}, \quad \mbox{for all } \gamma \in \Pi(P, Q).
    \end{align*}
    This yields that
    \begin{align*}
        S_{\varepsilon_1}(P, Q) \le S_{\varepsilon_2}(P, Q), \quad \mbox{for all } \varepsilon_1 \le \varepsilon_2,
    \end{align*}
    and thus \eqref{eq:eot_sandwich} follows.
    
    We then study the limit of $S_\varepsilon$ as $\varepsilon \rightarrow \infty$.
    By the assumption that $c$ is bounded or $P$ and $Q$ have compact support, there exists $M > 0$ such that $\sup_{\gamma \in \Pi(P, Q)} \int c d \gamma \le M < \infty$.
    As a result,
    \begin{align*}
        \sup_{\gamma \in \Pi(P, Q)} \abs{\frac{1}{\varepsilon} \int c d\gamma + \kl{\gamma}{P \otimes Q} - \kl{\gamma}{P \otimes Q}} \le \frac{M}{\varepsilon},
    \end{align*}
    which implies that
    \begin{align*}
        \inf_{\gamma \in \Pi(P, Q)} \left[ \frac{1}{\varepsilon} \int c d\gamma + \kl{\gamma}{P \otimes Q} \right] \rightarrow \inf_{\gamma \in \Pi(P, Q)} \kl{\gamma}{P \otimes Q} = 0, \quad \mbox{as } \varepsilon \rightarrow \infty.
    \end{align*}
    By the strict convexity of KL, the problem on the LHS has a unique minimizer $\gamma_\varepsilon$ and the problem on the RHS has a unique minimizer $\gamma_* = P \otimes Q$.
    Now, by the tightness of $\Pi(P, Q)$ (e.g., \cite[Theorem. 1.7]{santambrogio2015}), every sequence of $\{\gamma_\varepsilon\}$ has a weakly converging subsequence whose limit must be $\gamma_*$.
    Therefore, the claim \eqref{eq:a_eps_infinity} holds true.

    Let $c = c_1 \oplus c_2$.
    According to \eqref{eq:a_eps_infinity}, we have
    \begin{align*}
        \lim_{\varepsilon \rightarrow \infty} S_\varepsilon(\pxy, \prodxy) = (\pxy \otimes \prodxy)[c] = (\px \otimes \px) [c_1] + (\py \otimes \py)[c_2].
    \end{align*}
    Similarly, it holds that
    \begin{align*}
        \lim_{\varepsilon \rightarrow \infty} S_\varepsilon(\pxy, \pxy) &= (\px \otimes \px) [c_1] + (\py \otimes \py)[c_2] \\
        \lim_{\varepsilon \rightarrow \infty} S_\varepsilon(\prodxy, \prodxy) &= (\px \otimes \px) [c_1] + (\py \otimes \py)[c_2].
    \end{align*}
    Consequently, $\lim_{\varepsilon \rightarrow \infty} \eotic_\varepsilon(X, Y) = 0$.
    An analogous argument implies that, when $c = c_1 \otimes c_2$
    \begin{align*}
        &\lim_{\varepsilon \rightarrow \infty} \eotic_\varepsilon(X, Y) = \Expect_{\pxy} \left[ \Expect_{\px} [c_1(X, X') \mid X] \Expect_{\py}[c_2(Y, Y') \mid Y]\right] \\
        &\quad - \frac12 \Expect_{\pxy^2}[c_1(X, X') c_2(Y, Y')] - \frac12 \Expect_{(\prodxy)^2}[c_1(X, X') c_2(Y, Y')] = -\frac12 \hsic_{c_1, c_2}(X, Y).
    \end{align*}
    
    Note that
    \begin{align*}
        \lim_{\varepsilon \rightarrow 0} S_\varepsilon(P, Q) = S_0(P, Q)
    \end{align*}
    when both $P$ and $Q$ are densities \citep{leonard12} and when both of them are discrete measures \citep[Proposition 4.1]{peyre2019computational}.
    The statement \eqref{eq:etic_eps_zero} follows immediately from the fact that $S_0(P, P) = 0$ for all $P$.
\end{proof}

We then prove the validity of \teotic~as a dependence measure as stated in \Cref{prop:eotic_valid}.
\begin{proof}[Proof of \Cref{prop:eotic_valid}]
    Due to \citet[Lemma 5.2]{blanchard2011generalizing}, the Gibbs kernel
    \begin{align*}
        k_\varepsilon(z, z') := e^{-c(z, z') / \varepsilon} = k_1(x,x') k_2(y,y')
    \end{align*}
    is universal since both $k_x$ and $k_y$ are.
    It is also clear that $k_\varepsilon$ is positive since both $k_x$ and $k_y$ are.
    Consequently, the Sinkhorn divergence $\sdiv_\varepsilon$ defines a semi-metric on $\pspace(\calX \times \calY)$ according to \citet[Theorem 1]{feydy2019interpolating}.
    Hence, if $\pxy, \prodxy \in \pspace(\calX \times \calY)$, then $\eotic_\varepsilon(X, Y) := \sdiv_\varepsilon(\pxy, \prodxy) = 0$ iff $\pxy = \prodxy$.
\end{proof}

Finally, we analyze the computational complexity of the Tensor Sinkhorn algorithm for additive cost functions, i.e.,
\begin{align}\label{eq:a_weighted_quadratic_cost}
    c(z, z') := c_1(x, x') + c_2(y, y'),
\end{align}
where $z = (x, y)$ and $z' = (x', y')$.

Let $\{x_i\}_{i=1}^n$ and $\{y_j\}_{j=1}^n$ be two sets of atoms.
Note that the two sets are assumed to be of the same size for convenience.
Let $A$ and $B$ be two probability measures on $\{x_i\}_{i=1}^n \times \{y_j\}_{j=1}^n$.
For convenience, both $A$ and $B$ are represented as a matrix, i.e., $A_{ij} = A(x_i, y_j)$.
For instance, if we choose $A = \hpxy$ and $B = \hprodxy$, then, in its matrix form, $A = I_{n} / n$ and $B = \mathbf{1}_{n\times n} / n^2$.
Denote $C_1$ and $C_2$ as the cost matrices of $\{x_i\}_{i=1}^n$ and $\{y_j\}_{j=1}^n$, respectively.
Define Gibbs matrices $K_1 := e^{-C_1/\varepsilon}$ and $K_2 := e^{-C_2/\varepsilon}$, where the exponential function is applied element-wisely.
Let $K := K_2 \otimes K_1 \in \reals^{n^2 \times n^2}$ be the Gibbs matrix associated with the cost matrix on the pairs $\{(x_1, y_1), (x_2, y_1), \dots, (x_n, y_n)\}$, where $\otimes$ is the Kronecker product.

\begin{proposition}\label{prop:tensor_sinkhorn}
The Tensor Sinkhorn algorithm outputs an $\delta$-accurate estimate of the entropic cost $S(A, B)$ in $O\left( n^3 \log({\kappa_1 \kappa_2 \kappa_3})/ \delta \right)$ arithmetic operations, where $\kappa_1 := \max_{i, i'} k_{1}^{-1}(x_i, x_{i'})$, $\kappa_2 := \max_{j, j'} k_{2}^{-1}(y_j, y_{j'})$, and $\kappa_3 := \max_{i,j}\{a_{ij}^{-1}, b_{ij}^{-1}\}$.
\end{proposition}
\begin{proof}
    Let $a := \Vect(A) \in \reals^{n^2}$ and $b := \Vect(B) \in \reals^{n^2}$ be the probability vectors corresponding to $A$ and $B$, respectively.
    Denote $u := \Vect(U) \in \reals^{n^2}$ and $v := \Vect(V) \in \reals^{n^2}$.
    The Sinkhorn algorithm to solve $S_\varepsilon(a, b)$ has the following two update steps:
    \begin{align*}
        u = a \oslash Kv \quad \mbox{and} \quad v = b \oslash K^\top u.
    \end{align*}

    By the identity $\Vect(MNL) = (L^\top \otimes M) \Vect(N)$ for matrices $M$, $N$, and $L$ of compatible dimensions, we obtain
    \begin{align*}
        \Vect(K_1 V K_2^\top) = (K_2 \otimes K_1) \Vect(V) = K v.
    \end{align*}
    Thus, the update $U = A \oslash (K_1 V K_2^\top)$ is equivalent to $u = a \oslash Kv$.
    Similarly, the updated $V = B \oslash (K_1^\top U K_2)$ is equivalent to $v = b \oslash K^\top u$.
    Due to \citet[Theorem 1]{dvurechensky2018computational}, the Tensor Sinkhorn algorithm therefore outputs an $\delta$-accurate estimate in $O(\log({\kappa_1 \kappa_2 \kappa_3})/ \delta)$ iterations.
    Since each iteration costs $O(n^3)$ time, it has overall time complexity $O(n^3 \log({\kappa_1 \kappa_2 \kappa_3})/ \delta)$.
\end{proof}

\begin{remark}
    A direct application of the Sinkhorn algorithm leads to $O(n^4 \log({\kappa_1 \kappa_2 \kappa_3})/ \delta)$ time complexity, which is $n$ times slower than the Tensor Sinkhorn algorithm.
\end{remark}

We then characterize the convergence of the Tensor Sinkhorn algorithm with the random feature approximation as presented in \Cref{prop:rf_tensor_sinkhorn}.
\begin{proof}[Proof of \Cref{prop:rf_tensor_sinkhorn}]
    The proof is heavily inspired by \citet[Proof of Theorem 3.1]{scetbon2020linear}.
    In consideration of the space, we only present the part that is significantly different from theirs, i.e., a counterpart of \citet[Proposition 3.1]{scetbon2020linear}.
    This proposition gives a uniform tail bound for the ratio between the approximated kernel and the original kernel.
    In our case, we are approximating the kernel $K := K_2 \otimes K_1$ by $K_{\bm{u}, \bm{v}} := K_{2, \bm{v}} \otimes K_{1, \bm{u}}$.
    Hence, it suffices to bound
    \begin{align*}
        \sup_{x, x' \in \{x_i\}_{i=1}^n, y, y' \in \{y_i\}_{i=1}^n} \abs{\frac{k_{1, \bm{u}}(x, x') k_{2, \bm{v}}(y, y')}{k_{1}(x, x') k_{2}(y, y')} - 1}.
    \end{align*}
    Note that
    \begin{align*}
        \frac{k_{1, \bm{u}}(x, x')}{k_{1}(x, x')} = \frac1p \sum_{k=1}^p \frac{\varphi(x, u_k)^\top \varphi(x', u_k)}{k_1(x, x')}
    \end{align*}
    is a sum of nonnegative i.i.d.~random variables with mean 1.
    Due to \Cref{asmp:random_feature}, they are also bounded.
    It follows from the Hoeffding inequality that
    \begin{align*}
        \Prob\left( \abs{\frac{k_{1, \bm{u}}(x, x')}{k_{1}(x, x')} - 1} \ge t \right) \le 2 \exp\left( -\frac{pt^2}{C^2} \right).
    \end{align*}
    The same inequality holds for the ratio $k_{2, \bm{v}}(y, y') / k_2(y, y')$.
    Since
    \begin{align*}
        \abs{\frac{k_{1, \bm{u}}(x, x') k_{2, \bm{v}}(y, y')}{k_{1}(x, x') k_{2}(y, y')} - 1} \le \abs{\frac{k_{1, \bm{u}}(x, x')}{k_{1}(x, x')} - 1} \abs{\frac{k_{2, \bm{v}}(y, y')}{k_{2}(y, y')} - 1} + \abs{\frac{k_{1, \bm{u}}(x, x')}{k_{1}(x, x')} - 1} + \abs{\frac{k_{2, \bm{v}}(y, y')}{k_{2}(y, y')} - 1},
    \end{align*}
    it follows that
    \begin{align*}
        \Prob\left( \abs{\frac{k_{1, \bm{u}}(x, x') k_{2, \bm{v}}(y, y')}{k_{1}(x, x') k_{2}(y, y')} - 1} \le t^2 + 2t \right)
        &\ge \Prob\left( \left\{ \abs{\frac{k_{1, \bm{u}}(x, x')}{k_{1}(x, x')} - 1} \le t \right\} \bigcap \left\{ \abs{\frac{k_{2, \bm{v}}(y, y')}{k_{2}(y, y')} - 1} \le t \right\} \right) \\
        &= \Prob\left( \abs{\frac{k_{1, \bm{u}}(x, x')}{k_{1}(x, x')} - 1} \le t \right) \Prob\left( \abs{\frac{k_{2, \bm{v}}(y, y')}{k_{2}(y, y')} - 1} \le t \right) \\
        &\ge 1 - 4 \exp\left( -\frac{pt^2}{C^2} \right).
    \end{align*}
    Equivalently,
    \begin{align*}
        \Prob\left( \abs{\frac{k_{1, \bm{u}}(x, x') k_{2, \bm{v}}(y, y')}{k_{1}(x, x') k_{2}(y, y')} - 1} \ge t \right) \le 4 \exp\left( -\frac{p(\sqrt{t+1} - 1)^2}{C^2} \right).
    \end{align*}
    A uniform bound yields
    \begin{align*}
        \Prob\left( \sup_{x, x' \in \{x_i\}_{i=1}^n, y, y' \in \{y_i\}_{i=1}^n} \abs{\frac{k_{1, \bm{u}}(x, x') k_{2, \bm{v}}(y, y')}{k_{1}(x, x') k_{2}(y, y')} - 1} \ge t \right) \le 4n^4 \exp\left( -\frac{p(\sqrt{t+1} - 1)^2}{C^2} \right).
    \end{align*}
\end{proof}

\begin{remark}
    Let $\hat S_{\varepsilon, c_{\bm{u}, \bm{v}}}(A, B)$ be the cost computed from \Cref{alg:tensor_sinkhorn}.
    Following~\citet[Theorem 1]{dvurechensky2018computational}, we can get that
    \begin{align*}
        \abs{\hat S_{\varepsilon, c_{\bm{u}, \bm{v}}}(A, B) - S_{\varepsilon, c_{\bm{u}, \bm{v}}}(A, B)} \le \tau
    \end{align*}
    in $O\left( p n^2 \log({\kappa_1 \kappa_2 \kappa_3})/ \tau \right)$ arithmetic operations, where $\kappa_1 := \max_{i, i'} k_{1, \bm{u}}^{-1}(x_i, x_{i'})$, $\kappa_2 := \max_{j, j'} k_{2, \bm{v}}^{-1}(y_j, y_{j'})$, and $\kappa_3 := \max_{i,j}\{a_{ij}^{-1}, b_{ij}^{-1}\}$.
\end{remark}

%%%%%%%%%%%%%%%%%%%%%%%%%%%%%%%%%%%%%%%%%%%%%%%%%%%%%%%%%%%%%%%%%%%%%%%%
%%%%%%%%%%%%%%%%%%%%%%%%%%%%%%%%%%%%%%%%%%%%%%%%%%%%%%%%%%%%%%%%%%%%%%%%
\section{Consistency of \teotic}
\label{sec:consistency}
%%%%%%%%%%%%%%%%%%%%%%%%%%%%%%%%%%%%%%%%%%%%%%%%%%%%%%%%%%%%%%%%%%%%%%%%
%%%%%%%%%%%%%%%%%%%%%%%%%%%%%%%%%%%%%%%%%%%%%%%%%%%%%%%%%%%%%%%%%%%%%%%%

In this section, we prove the main results in \Cref{sec:thm}.
For the sake of generality, we start by considering the formulation in \eqref{eq:a_sinkhorn_div}.
We focus on the weighted quadratic cost function
\begin{align*}
    c(z, z') := w_1 \norm{x - x'}^2 + w_2 \norm{y - y'}^2,
\end{align*}
where $z = (x, y)$, $z' = (x', y')$ and $w_1, w_2 \in \reals_+$.
Denote $w := \max\{w_1, w_2\}$.
Due to \Cref{lem:reg_one}, we assume, w.l.o.g., that $\varepsilon = 1$ and write $S(P, Q) := S_1(P, Q)$.

%%%%%%%%%%%%%%%%%%%%%%%%%%%%%%%%%%%%%%%%%%%%%%%%%%%%%%%%%%%%%%%%%%%%%%%%
\subsection{Smoothness Properties of the Schr\"odinger Potentials}
\label{sub:smooth_potentials}
%%%%%%%%%%%%%%%%%%%%%%%%%%%%%%%%%%%%%%%%%%%%%%%%%%%%%%%%%%%%%%%%%%%%%%%%

We start by deriving some smoothness properties of the Schr\"odinger potentials.
Our proofs are deeply inspired by \citet{mena2019statistical}.
Our results generalize theirs to weighted quadratic cost functions.

\begin{assumption}\label{asmp:a_subg}
    We assume that $P_X$, $P_Y$, $Q_X$, and $Q_Y$ are all subG$(\sigma^2)$.
\end{assumption}

\begin{proposition}\label{prop:a_quadratic_bound}
  Under \Cref{asmp:a_subg}.
  there exist smooth Schr\"odinger potentials $(f, g)$ for $S(P, Q)$ such that the optimality conditions \eqref{eq:a_optim_potentials} hold for all $z, z' \in \reals^d$. Moreover, we have
  \begin{align*}
      f(z) &\ge -d\sigma^2 \left[ 2w_1 + 2w_2 + 4w_1^2 (\sqrt{2d_1}\sigma + \norm{x})^2 + 4w_2^2 (\sqrt{2d_2}\sigma + \norm{y})^2 \right] - 1 \\
      f(z) &\le w_1(\norm{x} + \sqrt{2d_1} \sigma)^2 + w_2(\norm{y} + \sqrt{2d_2} \sigma)^2,
  \end{align*}
  and for $g$ similarly.
\end{proposition}
\begin{proof}
  Let $(f_0, g_0)$ be a pair of Schr\"odinger potentials.
  Since $(f_0 + C, g_0 - C)$ is also a pair of Schr\"odinger potentials for any constant $C \in \reals$, we assume, w.l.o.g., that $P[f_0] = Q[g_0] = \frac12 S(P, Q) \ge 0$.
  Define
  \begin{align}\label{eq:a_sch_potentials}
      f(z) := -\log{\int e^{g_0(z') - c(z, z')} dQ(z')} \quad \mbox{and} \quad g(z') := -\log{\int e^{f(z) - c(z, z')} dP(z)}.
  \end{align}
  We claim that the pair $(f, g)$ satisfies the requirements.
  
  Since $(f_0, g_0)$ is a pair of Schr\"odinger potentials, it holds that
  \begin{align*}
      g_0(z') \txtover{a.s.}{=} -\log{\int e^{f_0(z) - c(z, z')}} dP(z) \le -P[f_0] + w_1 \Expect_{P_X}[\norm{X - x'}^2] + w_2 \Expect_{P_Y}[\norm{Y - y'}^2],
  \end{align*}
  by Jensen's inequality.
  Note that $P[f_0] \ge 0$ and, by \Cref{lem:subg_moment}, $\Expect_{P_X}[\norm{X}^2] \le 2d_1 \sigma^2$.
  It follows that
  \begin{align*}
      g_0(z') - c(z, z') \le w_1 \left[ 2d_1\sigma^2 + 2 \norm{x'} (\sqrt{2d_1}\sigma + \norm{x}) \right] + w_2 \left[ 2d_2\sigma^2 + 2 \norm{y'} (\sqrt{2d_2}\sigma + \norm{y}) \right],
  \end{align*}
  and thus
  \begin{align*}
      \int e^{g_0(z') - c(z, z')} dQ(z')
      &\le e^{2 (w_1 d_1 + w_2 d_2) \sigma^2} \left[ \int e^{4w_1 \norm{x'}(\sqrt{2d_1}\sigma + \norm{x})} dQ_X(x') \int e^{4w_2 \norm{y'}(\sqrt{2d_2}\sigma + \norm{y})} dQ_Y(y') \right]^{1/2} \\
      &\le 2 e^{2 (w_1 d_1 + w_2 d_2) \sigma^2} e^{4d_1 \sigma^2 w_1^2 (\sqrt{2d_1}\sigma + \norm{x})^2 + 4d_2 \sigma^2 w_2^2 (\sqrt{2d_2}\sigma + \norm{y})^2} < \infty, \quad \mbox{by \Cref{lem:subg_moment}}.
  \end{align*}
  Hence, $f(z)$ is well-defined for all $z \in \reals^d$.
  Moreover, we have the lower bound
  \begin{align*}
      f(z)
      &\ge -d_1\sigma^2 \left[ 2w_1 + 4w_1^2 (\sqrt{2d_1}\sigma + \norm{x})^2 \right] - d_2\sigma^2 \left[ 2w_2 + 4w_2^2 (\sqrt{2d_2}\sigma + \norm{y})^2 \right] - 1 \\
      &\ge -d\sigma^2 \left[ 4w + 4w_1^2 (\sqrt{2d_1}\sigma + \norm{x})^2 + 4w_2^2 (\sqrt{2d_2}\sigma + \norm{y})^2 \right] - 1
  \end{align*}
  For the upper bound, by Jensen's inequality, it holds that
  \begin{align*}
      f(z)
      &\le -Q[g_0] + w_1 \Expect_{Q_X}\norm{x - X'}^2 + w_2 \Expect_{Q_Y} \norm{y - Y'}^2 \\
      &\le w_1(\norm{x} + \sqrt{2d_1} \sigma)^2 + w_2(\norm{y} + \sqrt{2d_2} \sigma)^2.
  \end{align*}
  Similar arguments prove the claim for $g$.
  Now, it remains to show that $(f, g)$ satisfies the optimality conditions \eqref{eq:a_optim_potentials} for all $z, z \in \reals^d$.
  By definition, it is clear that
  \begin{align*}
      \int e^{f(z) + g(z') - c(z, z')} dP(z) = 1 \quad \mbox{and} \quad \int e^{f(z) + g_0(z') - c(z, z')} dQ(z') = 1, \quad \forall z, z' \in \reals^d.
  \end{align*}
  Since $(f_0, g_0)$ is a pair of Schr\"odinger potentials, we also have
  \begin{align*}
      \int e^{f_0(z) + g_0(z') - c(z, z')} dP(z) dQ(z') = 1.
  \end{align*}
  Consequently, by Jensen's inequality
  \begin{align*}
      &\quad \int (f - f_0) dP + \int (g - g_0) dQ \\
      &\ge -\log{\int e^{f_0 - f} dP} - \log{\int e^{g_0 - g} dQ} \\
      &= -\log{\int e^{f_0(z) + g_0(z') - c(z, z')} dP(z) dQ(z')} - \log{\int e^{f(z) + g_0(z') - c(z, z')} dP(z) dQ(z')} \\
      &= 0.
  \end{align*}
  Since both $(f_0, g_0)$ and $(f, g)$ are Schr\"odinger potentials, the above equality holds true.
  This implies that $\int (g_0 - g) dQ = \log{\int e^{g_0 - g} dQ}$, and thus $g = g_0 + C$ $Q$-almost surely by the strict concavity of $\log$.
  Therefore, we have
  \begin{align*}
      \int e^{f(z) + g(z') - c(z, z')} dQ(z') = e^C \int e^{f(z) + g_0(z') - c(z, z')} dQ(z') = e^C, \quad \forall z, z' \in \reals^d.
  \end{align*}
  Taking integrals with respect to $P$ implies that $C = 0$, which completes the proof.
\end{proof}

The next proposition shows that there exist Schr\"odinger potentials satisfying H\"older-type conditions.

\begin{definition}\label{def:f_sigma}
    For any $\sigma \in \reals_+$, $d \in \mathbb{N}_+$, and $w = (w_1, w_2) \in \reals_+^2$, let $\calF_\sigma := \calF_{\sigma, d, w}$ be the set of smooth functions such that, for any $k \in \bbN_+$ and any multi-index $\alpha$ with $\abs{\alpha} = k$,
    \begin{align}\label{eq:holder_small}
    \abs{D^\alpha\left( f(x, y) - w_1 \norm{x}^2 - w_2 \norm{y}^2 \right)} \le C_{k,d,w}
    \begin{cases}
      (1 + \sigma^4) & \mbox{if } k = 0 \\
      \sigma^{k}(1 + \sigma)^k & \mbox{otherwise},
    \end{cases}
  \end{align}
  if $\norm{z} \le \sqrt{d} \sigma$, and
  \begin{align}\label{eq:holder_large}
    \abs{D^\alpha\left( f(x, y) - w_1 \norm{x}^2 - w_2 \norm{y}^2 \right)} \le C_{k,d,w}
    \begin{cases}
      [1 + (1 + \sigma^2) \norm{x}^2] & \mbox{if } k = 0 \\
      \sigma^{k}(\sqrt{\sigma \norm{x}} + \sigma \norm{x})^k & \mbox{otherwise},
    \end{cases}
  \end{align}
  if $\norm{z} > \sqrt{d} \sigma$, where $C_{k,d,w}$ is a constant depending on $k$, $d$, and $w$.
\end{definition}

\begin{proposition}\label{prop:smooth_potentials}
  Under \Cref{asmp:a_subg},
  there exist Schr\"odinger potentials $(f, g)$ such that the optimality conditions \eqref{eq:a_optim_potentials} hold for all $z, z' \in \reals^d$ and $f, g \in \calF_\sigma$.
\end{proposition}
\begin{proof}
  Let $(f, g)$ be a pair of Schr\"odinger potentials satisfying the requirements in \Cref{prop:a_quadratic_bound}.
  Denote $\bar f(x, y) := f(x, y) - w_1 \norm{x}^2 - w_2 \norm{y}^2$.
  Note that
  \begin{align*}
      \bar f(z)
      &= -\log{e^{-\bar f(x, y)}} = -\log{\int e^{w_1 \norm{x}^2 + w_2 \norm{y}^2 + g(z') - c(z, z')} dQ(z')} \\
      &= -\log{\int e^{g(z') - w_1 \norm{x'}^2 - w_2 \norm{y'}^2 + 2w_1 \ip{x, x'} + 2w_2 \ip{y, y'}} dQ(z')}.
  \end{align*}
  The desired inequalities for $k=0$ follow directly from \Cref{prop:a_quadratic_bound}.
  We focus on $k > 0$.
  According to the multivariate Fa\'a di Bruno formula \citep{constantine1996multivariate}, we have
  \begin{align*}
      D^\alpha \bar f(z) = \sum_{\lambda_1 + \dots + \lambda_k = \alpha} C_{\alpha, \lambda_1, \dots, \lambda_k} \prod_{i=1}^k M_{\lambda_i},
  \end{align*}
  where
  \begin{align}\label{eq:cumulant}
      M_\lambda = \frac{\int (\tilde z')^\lambda \exp\left\{ g(z') - w_1 \norm{x'}^2 - w_2 \norm{y'}^2 + 2w_1 \ip{x, x'} + 2w_2 \ip{y, y'} \right\} dQ(z')}{\int \exp\left\{ g(z') - w_1 \norm{x'}^2 - w_2 \norm{y'}^2 + 2w_1 \ip{x, x'} + 2w_2 \ip{y, y'} \right\} dQ(z')}.
  \end{align}
  Here $\tilde z' = (2w_1 x'; 2w_2 y')$ and $z^\lambda = \prod_{i=1}^d z_i^{\lambda_i}$.
  By \Cref{lem:cumulant} below, it holds that
  \begin{align*}
      \abs{D^\alpha \bar f(z)} \le C_{k, d, w}
      \begin{cases}
          \sigma^k (1 + \sigma^k) & \mbox{if } \norm{z} \le \sqrt{d} \sigma \\
          \sigma^k (\sigma \norm{z} + \sqrt{\sigma \norm{z}})^k & \mbox{if } \norm{z} > \sqrt{d} \sigma,
      \end{cases}
  \end{align*}
  which proves the claim.
\end{proof}

\begin{lemma}\label{lem:cumulant}
    Recall $M_\lambda$ in \eqref{eq:cumulant}.
    Under \Cref{asmp:a_subg},
    for $\abs{\lambda} > 0$,
    we have
    \begin{align*}
        \abs{M_\lambda} \le C_{\abs{\lambda}, d, w}
        \begin{cases}
            \sigma^{\abs{\lambda}} (\sigma + \sigma^2)^{\abs{\lambda}} & \mbox{if } \norm{z} \le \sqrt{d} \sigma \\
            \sigma^{\abs{\lambda}} ( \sigma \norm{z} + \sqrt{\sigma \norm{z}} )^{\abs{\lambda}} & \mbox{if } \norm{z} > \sqrt{d} \sigma
        \end{cases}.
    \end{align*}
\end{lemma}
\begin{proof}
  We first bound the denominator.
  By the optimality conditions \eqref{eq:a_optim_potentials}, it holds that
  \begin{align*}
      &\quad \left( \int \exp\left\{ g(z') - w_1 \norm{x'}^2 - w_2 \norm{y'}^2 + 2w_1 \ip{x, x'} + 2w_2 \ip{y, y'} \right\} dQ(z') \right)^{-1} \\
      &= e^{f(x, y) - w_1 \norm{x}^2 - w_2 \norm{y}^2}
      \le e^{w_1(2d_1 \sigma^2 + 2\sqrt{2d_1} \sigma \norm{x}) + w_2(2d_2 \sigma^2 + 2\sqrt{2d_2} \sigma \norm{y})},
  \end{align*}
  where the last inequality follows from \Cref{prop:a_quadratic_bound}.
  To bound the numerator, we use the truncation technique.
  Let $A := \{ (x', y'): \norm{2w_1 x'} \le K, \norm{2w_2 y'} \le K \}$ for some constant $K$ to be determined later.
  On the set $A$, it is clear that $(\tilde z')^\lambda \le \norm{\tilde z'}^{\abs{\lambda}} \le K^{\abs{\lambda}}$, and thus
  \begin{align*}
      \frac{\int_A (\tilde z')^\lambda \exp\left\{ g(z') - w_1 \norm{x'}^2 - w_2 \norm{y'}^2 + 2w_1 \ip{x, x'} + 2w_2 \ip{y, y'} \right\} dQ(z')}{\int \exp\left\{ g(z') - w_1 \norm{x'}^2 - w_2 \norm{y'}^2 + 2w_1 \ip{x, x'} + 2w_2 \ip{y, y'} \right\} dQ(z')} \le K^{\abs{\lambda}}.
  \end{align*}
  On the set $A^c$, we proceed as follows.
  According to \Cref{prop:a_quadratic_bound}, we have
  \begin{align*}
      e^{g(x', y') - w_1 \norm{x'}^2 - w_2\norm{y'}^2} \le e^{w_1(2d_1 \sigma^2 + 2\sqrt{2d_1} \sigma \norm{x'}) + w_2(2d_2 \sigma^2 + 2\sqrt{2d_2} \sigma \norm{y'})},
  \end{align*}
  which yields
  \begin{align*}
      &\quad \int_{A^c} (\tilde z')^\lambda \exp\left\{ g(z') - w_1 \norm{x'}^2 - w_2 \norm{y'}^2 + 2w_1 \ip{x, x'} + 2w_2 \ip{y, y'} \right\} dQ(z') \\
      &\le e^{2(w_1d_1 + w_2d_2)\sigma^2} \left[ \int_{A^c} (\tilde z')^{2\lambda} dQ(z') \int_{A^c} e^{2w_1 \norm{x'}(\norm{x} + \sqrt{2d_1}\sigma) + 2w_2 \norm{y'}(\norm{y} + \sqrt{2d_2}\sigma)} dQ(z') \right]^{1/2}.
  \end{align*}
  For any $z' \in A^c$, we have either $\norm{2w_1 x'} > K$ or $\norm{2w_2 y'} > K$.
  If the former is true, then
  \begin{align*}
      \int_{A^c} (\tilde z')^{2\lambda} dQ(z') \le \int_{A^c} e^{-\frac{K^2}{16w_1^2 d_1\sigma^2}} e^{\frac{\norm{2w_1 x'}^2}{16w_1^2 d_1\sigma^2}} (\tilde z')^{2\lambda} dQ(z') \le C_{\abs{\lambda}, d, w} e^{-\frac{K^2}{16w^2 d\sigma^2}} \sigma^{2\abs{\lambda}},
  \end{align*}
  where $w = \max\{w_1, w_2\}$.
  The same bound holds if the latter is true.
  Furthermore, by the Cauchy-Schwartz inequality and \Cref{lem:subg_moment} in \Cref{sec:Technical_Lemmas}, we have
  \begin{align*}
      \int_{A^c} e^{2w_1 \norm{x'}(\norm{x} + \sqrt{2d_1}\sigma) + 2w_2 \norm{y'}(\norm{y} + \sqrt{2d_2}\sigma)} dQ(z') \le e^{4w_1^2 d_1 \sigma^2 (\norm{x} + \sqrt{2d_1}\sigma)^2 + 4w_2^2 d_2 \sigma^2 (\norm{y} + \sqrt{2d_2}\sigma)^2}
  \end{align*}
  Putting all together, we get
  \begin{align*}
      &\quad \frac{\int_{A^c} (\tilde z')^\lambda \exp\left\{ g(z') - w_1 \norm{x'}^2 - w_2 \norm{y'}^2 + 2w_1 \ip{x, x'} + 2w_2 \ip{y, y'} \right\} dQ(z')}{\int \exp\left\{ g(z') - w_1 \norm{x'}^2 - w_2 \norm{y'}^2 + 2w_1 \ip{x, x'} + 2w_2 \ip{y, y'} \right\} dQ(z')} \\
      &\le C_{\abs{\lambda}, d, w} e^{-\frac{K^2}{32 w^2 d\sigma^2}} e^{2w_1^2 d_1 \sigma^2 (\norm{x} + \sqrt{2d_1}\sigma)^2 + 2w_2^2 d_2 \sigma^2 (\norm{y} + \sqrt{2d_2}\sigma)^2} \sigma^{\abs{\lambda}} \\
      &\le C_{\abs{\lambda}, d, w} e^{-\frac{K^2}{32 w^2 d\sigma^2}} e^{2w^2 d \sigma^2 [(\norm{x} + \sqrt{2d}\sigma)^2 + (\norm{y} + \sqrt{2d}\sigma)^2]} \sigma^{\abs{\lambda}}
  \end{align*}
  When $\norm{z} \le \sqrt{d} \sigma$, it holds that $\norm{x} \le \sqrt{2d} \sigma$ and $\norm{y} \le \sqrt{2d} \sigma$.
  Hence, if we choose $K^2 = C_{\abs{\lambda}, d, w}(\sigma^4 + \sigma^6)$ for some sufficiently large constant $C_{\abs{\lambda}, d, w}$, then we have
  \begin{align*}
      \abs{M_{\lambda}} \le C_{\abs{\lambda}, d, w} \sigma^{\abs{\lambda}} (\sigma + \sigma^2)^{\abs{\lambda}}.
  \end{align*}
  When $\norm{z} > \sqrt{d} \sigma$, if we choose $K^2 = C_{\abs{\lambda}, d, w} (\sigma^4 \norm{z}^2 + \sigma^3 \norm{z})$, then we have
  \begin{align*}
      \abs{M_{\lambda}} \le C_{\abs{\lambda}, d, w} \sigma^{\abs{\lambda}} \left( \sigma \norm{z} + \sqrt{\sigma \norm{z}} \right)^{\abs{\lambda}}.
  \end{align*}
\end{proof}

When $P$ and $Q$ have bounded support, we can further show that the Schr\"odinger potentials can be chosen to be bounded.
\begin{proposition}\label{prop:bounded_potential}
  Assume that $P$ and $Q$ are supported on a bounded domain of radius $D$.
  Then there exist Schr\"odinger potentials $(f, g)$ such that 1) the optimality conditions \eqref{eq:a_optim_potentials} hold for all $x, y \in \reals^d$ and 2) $\norm{f}_\infty \le 8wD^2$ and $\norm{g}_{\infty} \le 8wD^2$.
\end{proposition}
\begin{proof}
  Let $(f, g)$ the Schr\"odinger potentials defined in \eqref{eq:a_sch_potentials}.
  By the proof of \Cref{prop:a_quadratic_bound}, they satisfy \eqref{eq:a_optim_potentials} everywhere.
  Moreover, we have
  \begin{align*}
      f(z) \le w_1 \Expect_{Q_X}\norm{x - X'}^2 + w_2 \Expect_{Q_Y} \norm{y - Y'}^2 \le 8wD^2
  \end{align*}
  and $g$ similarly.
\end{proof}

%%%%%%%%%%%%%%%%%%%%%%%%%%%%%%%%%%%%%%%%%%%%%%%%%%%%%%%%%%%%%%%%%%%%%%%%
\subsection{Controlling the Empirical Process and the U-Process}
\label{sub:process}
%%%%%%%%%%%%%%%%%%%%%%%%%%%%%%%%%%%%%%%%%%%%%%%%%%%%%%%%%%%%%%%%%%%%%%%%

We then upper bound the $\lone$ loss $\Expect\abs{T_n(X, Y) - T(X, Y)}$ by empirical processes and U-processes.
\begin{proposition}[Corollary 2 \citep{mena2019statistical}]
\label{prop:upper_bound_emp}
  Let $P, Q, P', Q' \in \pspace(\reals^d)$ be subG$(\sigma^2)$.
  Then we have
  \begin{align*}
    \abs{S(P', Q') - S(P, Q)} \le \sup_{f \in \calF_\sigma} \abs{\int f (dP' - dP)} + \sup_{g \in \calF_\sigma} \abs{\int g (dQ' - dQ)},
  \end{align*}
  where $\calF_\sigma$ is defined in \Cref{def:f_sigma}.
\end{proposition}

To simply the function class $\calF_\sigma$, we show in \Cref{lem:subg_func} in \Cref{sec:Technical_Lemmas} that $(1 + \sigma^{3s})^{-1} \calF_\sigma \subset \calF^s$ for $\calF^s$ defined below.
Consequently, we can separate the sub-Gaussian parameter $\sigma$ from the function class $\calF_\sigma$.

\begin{definition}\label{def:fs}
  For any $s \ge 2$, $d \in \mathbb{N}_+$, and $w = (w_1, w_2) \in \reals_+^2$, let $\calF^s := \calF^{s,d,w}$ be the set of functions satisfying
  \begin{align*}
    \abs{f(z)} &\le C_{s,d,w} (1 + \norm{z}^2) \\
    \abs{D^\alpha f(z)} &\le C_{s,d,w} (1 + \norm{z}^{\abs{\alpha}}), \quad \forall 1 \le \abs{\alpha} \le s,
  \end{align*}
  where $C_{s,d,w}$ is a constant depending on $s$, $d$, and $w$.
\end{definition}

In order to handle the U-process, we also need a variant function class of $\calF^s$ which we also define below.
\begin{definition}\label{def:fs_sigma}
  For any $\sigma \in \reals_+$, $s \ge 2$, $d \in \mathbb{N}_+$, and $w = (w_1, w_2) \in \reals_+^2$, let $\calF_\sigma^s := \calF_\sigma^{s,d,w}$ be the set of functions satisfying
  \begin{align*}
    \abs{f(z)} &\le C_{s,d,w} (1 + \max\{\norm{z}^2, \sigma^2\}) \\
    \abs{D^\alpha f(z)} &\le C_{s,d,w} (1 + \max\{\norm{z}^{\abs{\alpha}}, \sigma^{\abs{\alpha}}\}), \quad \forall 1 \le \abs{\alpha} \le s,
  \end{align*}
  where $C_{s,d,w}$ is a constant depending on $s$, $d$, and $w$.
\end{definition}

Let us control the complexity of $\calF^s$ and $\calF_\sigma^s$, which is achieved by the following covering number bound.
\begin{proposition}\label{prop:cover_number}
  Let $P \in \pspace(\reals^d)$ be subG$(\sigma^2)$.
  Let $\{Z_i\}_{i=1}^n \txtover{i.i.d.}{\sim} P$ and $\hat P_n$ be the empirical measure.
  There exists a random variable $L \ge 1$ depending on the sample $\{Z_i\}_{i=1}^n$ with $\Expect[L] \le 2$ such that
  \begin{align*}
    \log{N(\tau, \calF^s, \ltwo(\hat P_n))} \le C_{s,d,w} \tau^{-d/s} L^{d/2s} (1 + \sigma^{2d}) \quad \mbox{and} \quad \max_{f \in \calF^s} \norm{f}^2_{\ltwo(\hat P_n)} \le C_{s,d,w}(1 + L \sigma^4).
  \end{align*}
  Moreover, the same bounds hold for $\calF_\sigma^{s}$.
\end{proposition}
\begin{proof}[Proof of \Cref{prop:cover_number}]
  Define $L := \frac1n \sum_{i=1}^n e^{\norm{Z_i}^2 / 2d\sigma^2} \ge 1$.
  By the sub-Gaussianity of $P$, we have $\Expect[L] \le 2$.
  In order to apply \cite[Corollary 2.7.4]{vaart1996weak}, we partition $\reals^d$ into $\cup_{j \ge 1} B_j$ where $B_1 := [-\sigma, \sigma]^d$ and $B_j := [-j \sigma, j \sigma]^d \backslash [-(j-1)\sigma, (j-1)\sigma]^d$ for $j \ge 2$.
  Since $B_j$ is \emph{not convex} for $j \ge 2$, we further partition it into disjoint hypercubes $\{B_{j,k}\}_{k=1}^{2d}$, e.g.,
  \begin{align*}
    B_{j, 1} = [(j-1)\sigma, j\sigma] \times [-j\sigma, j\sigma]^{d-1}.
  \end{align*}

  Take any $j \ge 2$ and $k \in [2d]$.
  Firstly, it holds that
  \begin{align*}
    \lambda\{x: d(x, B_{j,k}) \le 1\} \le (\sigma + 2)(2j\sigma + 2)^{d-1} \le C_d (1 + j^d \sigma^d),
  \end{align*}
  where $\lambda$ is the Lebesgue measure.
  Secondly, the mass that $\hat P_n$ assigns to $B_{j,k}$ can be bounded as follows:
  \begin{align}\label{eq:assign_prob_partition}
    \hat P_n(Z \in B_{j,k}) \le \hat P_n\left( \norm{Z}^2 > d \sigma^2 (j-1)^2 \right) \le \hat P_n \left[ e^{\norm{Z}^2/2d\sigma^2} \right] e^{-(j-1)^2/2} = L e^{-(j-1)^2/2}.
  \end{align}
  Finally, we prove that $\calF^s \subset \calC_{M}^s(B_{j,k})$ with $M = C_{s,d,w} (1 + j^s\sigma^s)$, where $\calC_M^s(B_{j,k})$ is the set of continuous functions satisfying
  \begin{align*}
    \norm{f}_s := \max_{\abs{\alpha} \le s} \sup_{z \in B_{j,k}} \abs{D^\alpha f(z)} + \max_{\abs{\alpha} = s} \sup_{z, w \in B_{j,k}} \abs{D^\alpha f(z) - D^\alpha f(w)} \le M.
  \end{align*}
  In fact, for any $f \in \calF^s$, we have
  \begin{align*}
    \max_{\abs{\alpha} \le s} \sup_{z \in B_{j,k}} \abs{D^\alpha f(z)} \le C_{s,d,w} \sup_{z \in B_{j,k}} (1 + \norm{z}^s) \le C_{s,d,w} (1 + j^s \sigma^s),
  \end{align*}
  and
  \begin{align*}
    \max_{\abs{\alpha} = s} \sup_{z, w \in B_{j,k}} \abs{D^\alpha f(z) - D^\alpha f(w)} \le 2\max_{\abs{\alpha} = s} \sup_{z \in B_{j,k}} \abs{D^\alpha f(z)} \le C_{s,d} (1 + j^s\sigma^s).
  \end{align*}
  Note that the same argument holds for any $f \in \calF^s_\sigma$ since we can simply replace $1 + \norm{z}^s$ by $1 + \max\{\norm{z}^s, \sigma^s\}$.
  Now, applying \cite[Corollary 2.7.4]{vaart1996weak} with $r = 2$ and $V = d/s$ leads to
  \begin{align*}
    \log{N(\tau, \calF^s, \ltwo(\hat P_n))}
    &\le C_{s,d,w} \tau^{-d/s} L^{d/2s} \left( 1 + \sum_{j=2}^\infty \sum_{k=1}^{2d} (1 + j^d \sigma^d)^{\frac{2s}{d+2s}} (1 + j^s\sigma^s)^{\frac{2d}{d+2s}} e^{-\frac{d(j-1)^2}{d+2s}} \right)^{\frac{d+2s}{2s}} \\
    &\le C_{s,d,w} \tau^{-d/s} L^{d/2s} (1 + \sigma^{2d}) \left( 2d \sum_{j=1}^\infty j^{\frac{4ds}{d+2s}} e^{-\frac{d(j-1)^2}{d+2s}} \right)^{\frac{d+2s}{2s}} \\
    &\le C_{s,d,w} \tau^{-d/s} L^{d/2s} (1 + \sigma^{2d}), \quad \mbox{by the summability}.
  \end{align*}

  To verify the second inequality, we obtain
  \begin{align}\label{eq:max_norm_bound}
    \max_{f \in \calF^s} \norm{f}_{\ltwo(\hat P_n)}^2 = \max_{f \in \calF^s} \hat P_n [\abs{f(Z)}^2] \le C_{s,d,w} \hat P_n[(1 + \norm{Z}^4)].
  \end{align}
  Note that $\norm{Z}^4 \le C_d e^{\norm{Z}^2/2d\sigma^2} \sigma^4$.
  It follows that $\hat P_n [\norm{Z}^4] \le C_d L \sigma^4$, and thus
  \begin{align*}
    \max_{f \in \calF^s} \norm{f}_{\ltwo(\hat P_n)}^2
    \le C_{s,d,w} (1 + L\sigma^4).
  \end{align*}
  Again, the same argument hold for $\calF^s_\sigma$ by replacing $\norm{Z}^4$ with $\max\{\norm{Z}^4, \sigma^4\}$.
\end{proof}

With this covering number bound at hand, we can control the empirical process by the metric entropy.
\begin{proposition}\label{prop:emp_bound}
  Let $P \in \pspace(\reals^d)$ be subG$(\sigma^2)$.
  Let $\{Z_i\}_{i=1}^n \txtover{i.i.d.}{\sim} P$ and $\hat P_n$ be the empirical measure.
  Then,
  \begin{align*}
    \Expect\norm{\hat P_n - P}_{\calF^s}^2 \le C_{s,d,w} (1 + \sigma^{2d+4}) \frac1n, \quad \mbox{for all } s > d/2.
  \end{align*}
  Moreover, the same bound holds for $\calF_\sigma^s$.
\end{proposition}
\begin{proof}
  Define the symmetrized version of $\norm{\hat P_n - P}_{\calF^s}$ by
  \begin{align}
    \norm{\hat{\mathbb{S}}_n}_{\calF^s} := \sup_{f \in \calF^s} \abs{\frac1n \sum_{i=1}^n \varepsilon_i f(Z_i)},
  \end{align}
  where $\{\varepsilon_i\}_{i=1}^n$ are i.i.d.~Rademacher random variables that are independent with $\{Z_i\}_{i=1}^n$.
  According to \cite[Proposition 4.11]{wainwright2019high}, it holds that
  \begin{align*}
    \Expect \norm{\hat P_n - P}_{\calF^s}^2 \le 4\Expect \norm{\hat{\mathbb{S}}_n}_{\calF^s}^2.
  \end{align*}
  Conditioning on $\{Z_i\}_{i=1}^n$,
  the random variable $Z(f) := \frac1{\sqrt{n}} \sum_{i=1}^n \varepsilon_i f(Z_i)$ is a linear combination of independent Rademacher random variables.
  Hence, $Z(f)$ is a sub-Gaussian process (see \Cref{def:subg_process}) with respect to
  \begin{align*}
    \norm{f - g}_{\ltwo(\hat P_n)} = \sqrt{\frac1n \sum_{i=1}^n [f(Z_i) - g(Z_i)]^2}.
  \end{align*}
  It then follows from \Cref{prop:metric_entropy_integral} below that
  \begin{align*}
    \Expect_\varepsilon \sup_{f \in \calF^s} \abs{Z(f)}^2
    &\le C \left( \int_0^{2 \max_{f \in \calF^s} \norm{f}_{\ltwo(\hat P_n)}} \sqrt{\log{N(\tau, \calF^s, \ltwo(\hat P_n))}} d\tau \right)^2 \\
    &\le C_{s,d,w} \left( \int_0^{C_{s,d}\sqrt{1 + L\sigma^4}} \tau^{-d/2s} L^{d/4s} \sqrt{1 + \sigma^{2d}} d\tau \right)^2, \quad \mbox{by \Cref{prop:cover_number}} \\
    &= C_{s,d,w} (1 + \sigma^{2d}) L^{d/2s} (1 + L \sigma^4)^{1-d/2s}, \quad \mbox{by } s > d/2 \\
    &\le C_{s,d,w}(1 + \sigma^{2d+4}) L, \quad \mbox{by } L \ge 1.
  \end{align*}
  Note that $\Expect \norm{\hat{\mathbb{S}}_n}_{\calF^s}^2 = \frac1n \Expect \sup_{f \in \calF^s} \abs{Z(f)}^2$.
  Consequently, we have
  \begin{align}
    \Expect \norm{\hat P_n - P}_{\calF^s}^2 \le C_{s,d,w}(1 + \sigma^{2d+4}) \frac1n. 
  \end{align}
  The same argument holds for $\calF_{\sigma}^s$ since \Cref{prop:cover_number} holds true for $\calF_{\sigma}^s$.
\end{proof}

The following proposition controls the $\ltwo$ norm of the supremum of a sub-Gaussian process.
It can be obtained from \citet[Exercise 2.3.1]{gine2015mathematical}.
We give its proof here for self-completeness.
\begin{proposition}\label{prop:metric_entropy_integral}
  Let $\{Z(\theta)\}_{\theta \in \Theta}$ be a sub-Gaussian process with respect to a metric $\rho$ in $\Theta$ such that $\int_0^\infty \sqrt{\log{N(\tau, \Theta, \rho)}} d\tau < \infty$.
  Then it holds that, for any separable version of $Z$,
  \begin{align}
    \norm{\sup_{\theta \in \Theta} \abs{Z(\theta)}}_{\ltwo} \le \norm{Z(\theta_0)}_{\ltwo} + C \int_0^D \sqrt{\log{N(\tau, \Theta, \rho)}} d\tau,
  \end{align}
  where $\theta_0 \in \Theta$ is arbitrary and $D$ is the $\rho$-diameter of $\Theta$.
\end{proposition}
\begin{proof}
  Due to the separability, it suffices to prove
  \begin{align}
      \norm{\sup_{\theta \in \Theta'} \abs{Z(\theta)}}_{\ltwo} \le \norm{Z(\theta_0)}_{\ltwo} + C \int_0^D \sqrt{\log{N(\tau, \Theta, \rho)}} d\tau
  \end{align}
  for any finite $\Theta' \subset \Theta$.
  When the diameter $D = 0$, the claim holds trivially and thus we only need to focus on the case when $\abs{\Theta'} \ge 2$.
  By considering $(Z(\theta) - Z(\theta_0)) / (1 + \delta)D$ and $\rho / (1 + \delta)D$ instead of $Z(\theta)$ and $\rho$ for some any small $\delta > 0$, we may assume that $Z(\theta_0) = 0$ and $D \in (1/2, 1)$.
  Our proof relies on the classical chaining argument.
  
  \emph{Step 1. Construct a chain of projections.}
  Let $r_1 \in \bbN$ be such that, for any $\theta \in \Theta$, the ball $B(\theta, 2^{-r_1})$ centered at $\theta$ of radius $2^{-r_1}$ contains at most 1 element in $\Theta'$.
  Denote $\Theta_{r_1} := \Theta'$ and $\Theta_0 := \{\theta_0\}$.
  For each $1 \le r < r_1$, we take a $2^{-r}$ covering of $\Theta$ and let $\Theta_r$ be the collection of these centers.
  By definition, we get $\abs{\Theta_{r}} \le N(2^{-r}, \Theta, \rho)$ for all $0 \le r \le r_1$.
  For each $\theta \in \Theta'$, we construct a chain $(\pi_{r_1}(\theta), \pi_{r_1-1}(\theta), \dots, \pi_{0}(\theta))$ such that $\pi_{r}(\theta) \in \Theta_r$ as follows.
  For $r = r_1$, we let $\pi_{r}(\theta) = \theta$.
  For any $0 \le r < r_1$, we define $\pi_r(\theta)$ to be a point in $\Theta_r$ for which the ball $B(\pi_r(\theta), 2^{-r})$ contains $\pi_{r+1}(\theta)$.
  Note that there may be multiple points satisfying this requirement, but we select the same one for $\theta$ and $\theta'$ as long as $\pi_{r+1}(\theta) = \pi_{r+1}(\theta')$.
  
  \emph{Step 2. Telescoping.}
  By the triangle inequality, we have
  \begin{align*}
      \norm{\max_{\theta \in \Theta'} \abs{Z(\theta)}}_{\ltwo}
      &= \norm{\max_{\theta \in \Theta'} \abs{Z(\pi_{r_1}(\theta)) - Z(\pi_{0}(\theta))}}_{\ltwo}
      \le \sum_{r=1}^{r_1} \norm{\max_{\theta \in \Theta'} \abs{Z(\pi_{r}(\theta)) - Z(\pi_{r-1}(\theta))}}_{\ltwo}.
  \end{align*}
  Note that
  \begin{align*}
      \abs{\left\{ (\pi_r(\theta), \pi_{r-1}(\theta)): \theta \in \Theta' \right\}}
      &= \abs{\{ \pi_r(\theta): \theta \in \Theta' \}}
      \le \abs{\Theta_r} \le N(2^{-r}, \Theta, \rho).
  \end{align*}
  According to \cite[Lemma 2.3.3]{gine2015mathematical}, we obtain
  \begin{align*}
      \norm{\max_{\theta \in \Theta'} \abs{Z(\pi_r(\theta)) - Z(\pi_{r-1}(\theta))}}_{\ltwo}
      &\le C \sqrt{\log{N(2^{-r}, \Theta, \rho)}} \max_{\theta \in \Theta'} \norm{Z(\pi_r(\theta)) - Z(\pi_{r-1}(\theta))} \\
      &\le C 2^{-r+1} \sqrt{\log{N(2^{-r}, \Theta, \rho)}}.
  \end{align*}
  Consequently, it holds that
  \begin{align*}
      \norm{\max_{\theta \in \Theta'} \abs{Z(\theta)}}_{\ltwo}
      &\le C \sum_{r=1}^{r_1} 2^{-r+1} \sqrt{\log{N(2^{-r}, \Theta, \rho)}} \le C \int_0^{1} \sqrt{\log{N(\tau, \Theta, \rho)}} d\tau,
  \end{align*}
  which completes the proof.
\end{proof}

%%%%%%%%%%%%%%%%%%%%%%%%%%%%%%%%%%%%%%%%%%%%%%%%%%%%%%%%%%%%%%%%%%%%%%%%
\subsection{Proofs of Main Results}
\label{sub:proofs_main}
%%%%%%%%%%%%%%%%%%%%%%%%%%%%%%%%%%%%%%%%%%%%%%%%%%%%%%%%%%%%%%%%%%%%%%%%

We now prove the main consistency results in \Cref{sec:thm}.
For simplicity of the notation, we focus on the quadratic cost function, i.e., $w_1 = w_2 = 1$, and drop the dependency on $w$ (e.g., we write $C_{s,d} = C_{s, d, w}$.
The proofs can be adapted to weighted quadratic costs with minor modifications.
Let $\px \in \pspace(\reals^{d_1})$ and $\py \in \pspace(\reals^{d_2})$ with $d := d_1 + d_2$.
Suppose that $\{(X_i, Y_i)\}_{i=1}^n$ is an i.i.d.~sample from some joint distribution $\pxy$ with marginals $\px$ and $\py$, where $\pxy$ may or may not equal $\prodxy$.
Let $\hat P_n$ and $\hat Q_n$ be the empirical measures of $\{X_i\}_{i=1}^n$ and $\{Y_i\}_{i=1}^n$, respectively.

\begin{proof}[Proof of \Cref{prop:prod_emp_bound}]
  \emph{Step 1. Decoupling.}
  Due to the degeneracy, it suffices to bound
  \begin{align}\label{eq:U_process}
      \Expect\norm{\hprodxy}_\calF^2 = \Expect\left[ \sup_{f \in \calF} \abs{\frac{1}{n^2} \sum_{i,j=1}^n f(X_i, Y_j)}^2 \right].
  \end{align}
  We prove in the following that it boils down to control \eqref{eq:U_process} under the product measure $\prodxy$.
  When $\pxy = \prodxy$, the claim holds trivially.
  When $\pxy \neq \prodxy$, we use the decoupling technique \citep{pena1999decoupling}.
  Note that, by the Cauchy-Schwarz inequality,
  \begin{align*}
    \Expect\left[ \sup_{f \in \calF} \abs{\frac1{n^2} \sum_{i, j=1}^n f(X_i, Y_j)}^2 \right]
    \le C \Expect\left[ \sup_{f \in \calF} \abs{\frac1{n^2} \sum_{i \neq j}^n f(X_i, Y_j)}^2 + \sup_{f \in \calF} \abs{ \frac1{n^2}\sum_{i=1}^n f(X_i, Y_i)}^2 \right].
  \end{align*}
  Note that the second term on the RHS is a lower order term and can be taken care of by \Cref{prop:emp_bound}.
  Hence, it suffices to upper bound the first term.
  Let $\{\varepsilon_i\}_{i=1}^n$ be i.i.d.~Rademacher random variables and $\{(X_i', Y_i')\}_{i=1}^n$ be an independent copy of $\{(X_i, Y_i)\}_{i=1}^n$.
  Define
  \begin{align*}
      A_i := \begin{cases} X_i & \mbox{if } \varepsilon_i = 1 \\ X_i' & \mbox{if } \varepsilon_i = -1 \end{cases} \quad \mbox{and} \quad
      B_i := \begin{cases} Y_i' & \mbox{if } \varepsilon_i = 1 \\ Y_i & \mbox{if } \varepsilon_i = -1 \end{cases}.
  \end{align*}
  For any functional $F: \calF \rightarrow \reals_+$,
  let $\Phi(F) := \sup_{f \in \calF} F(f)^2$.
  For instance, we define
  \begin{align*}
    U_{X, Y}(f) := \frac1{n^2} \abs{\sum_{i \neq j} f(X_i, Y_j)}.
  \end{align*}
  It is clear that $\Phi$ is convex and increasing, and the target reads
  \begin{align*}
      \Expect\left[ \Phi(U_{X, Y}) \right]
      &= \Expect\left[ \Phi\left( \abs{\frac1{n^2} \sum_{i \neq j} \Expect\left[ f(X_i, Y_j) + f(X_i', Y_j) + f(X_i, Y_j') + f(X_i', Y_j') \mid \mathcal{Z}\right]} \right) \right],
  \end{align*}
  where $\mathcal{Z} := \{(X_i, Y_i)\}_{i=1}^n$.
  Since, for any $i \neq j$,
  \begin{align*}
      f(X_i, Y_j) + f(X_i', Y_j) + f(X_i, Y_j') + f(X_i', Y_j') = 4\Expect\left[ f(A_i, B_j) \mid \calZ, \calZ' \right],
  \end{align*}
  it follows from the convexity and the monotonicity of $\Phi$ that
  \begin{align*}
      \Expect\left[ \Phi(U_{X, Y}) \right] \le \Expect\left[ \Phi(4U_{A, B}) \right].
  \end{align*}
  Finally, the joint distribution of $(X_1, \dots, X_n, Y_1', \dots, Y_n')$ is the same as $(A_1, \dots, A_n, B_1, \dots, B_n)$, so we have
  \begin{align*}
      \Expect\left[ \Phi(U_{X, Y}) \right] \le \Expect\left[ \Phi(4U_{X, Y'}) \right].
  \end{align*}
  Adding back the diagonal terms proves the claim since $(X_i, Y_i') \sim \prodxy$.
  
  \emph{Step 2. Randomization.}
  We work under the measure $\pxy = \prodxy$.
  Note that
  \begin{align*}
    &\quad \Expect\left[ \sup_{f \in \calF} \abs{\frac1{n^2} \sum_{i,j=1}^n f(X_i, Y_j)}^2 \right] \\
    &= \Expect_Y \Expect_X\left[ \sup_{f \in \calF} \abs{\frac1{n^2} \sum_{i=1}^n \left[ \sum_{j=1}^n f(X_i, Y_j) - \Expect_{X'}\Big[ \sum_{j=1}^n \bar f(X_i', Y_j) \Big] \right]}^2 \right], \quad \mbox{by \eqref{eq:degeneracy}} \\
    &\le \Expect_Y \Expect_{X, X'} \left[ \sup_{f \in \calF} \abs{\frac1{n^2} \sum_{i=1}^n \left[ \sum_{j=1}^n f(X_i, Y_j) - \sum_{j=1}^n \bar f(X_i', Y_j) \right]}^2 \right], \quad \mbox{by Jensen's inequality} \\
    &= \Expect_Y \Expect_{X, X', \varepsilon} \left[ \sup_{f \in \calF} \abs{\frac1{n^2} \sum_{i=1}^n \varepsilon_i \left[ \sum_{j=1}^n \bar f(X_i, Y_j) - \sum_{j=1}^n \bar f(X_i', Y_j) \right]}^2 \right] \\
    &\le C \Expect\left[ \sup_{f \in \calF} \abs{\frac1{n^2} \sum_{i=1}^n \sum_{j=1}^n \varepsilon_i f(X_i, Y_j)}^2 \right], \quad \mbox{by the Cauchy-Schwarz inequality}.
  \end{align*}
  Repeating above arguments gives
  \begin{align*}
    \Expect\left[ \sup_{f \in \calF} \abs{\frac1{n^2} \sum_{i,j=1}^n f(X_i, Y_j)}^2 \right]
    &\le C \Expect\left[ \sup_{f \in \calF} \abs{\frac1{n^2} \sum_{i,j=1}^n \varepsilon_i \varepsilon_j' f(X_i, Y_j)}^2 \right] \\
    &\le C\Expect\left[ \sup_{f \in \calF} \abs{\frac1{n^2} \sum_{i,j=1}^n \varepsilon_i \varepsilon_j' f(X_i, Y_j)}^2 \right],
  \end{align*}
  where the last inequality follows from the Cauchy-Schwarz inequality and Jensen's inequality.
  Hence, it suffices to bound
  \begin{align*}
    A := \Expect\sup_{f \in \calF} \abs{\frac1{n^2} \sum_{i,j=1}^n \varepsilon_i \varepsilon_j' f(X_i, Y_j)}^2.
  \end{align*}

  \emph{Step 3. Metric entropy.}
  Define the process $Z(f) := \frac1{n^{3/2}} \sum_{i,j=1}^n \varepsilon_i \varepsilon_j' f(X_i, Y_j)$ for any $f \in \calF$.
  We claim that it is a sub-Gaussian process with respect to
  \begin{align}
    \norm{f - g}_{\ltwo(\hat P_n \otimes \hat Q_n)} = \sqrt{\frac1{n^2} \sum_{i,j=1}^n [f(X_i, Y_j) - g(X_i, Y_j)]^2}.
  \end{align}
  To prove it, let us control the moment generating function of the increment $Z(f) - Z(g)$.
  Denote $a_i := \sum_{j=1}^n \varepsilon_j' [f(X_i, Y_j) - g(X_i, Y_j)]$.
  Conditioning on $\{X_i, Y_i, \varepsilon_i'\}_{i=1}^n$,
  \begin{align*}
    Z(f) - Z(g) = \frac1{n^{3/2}} \sum_{i=1}^n a_i \varepsilon_i
  \end{align*}
  is a linear combination of independent Rademacher random variables.
  Consequently,
  \begin{align}
    \Expect_{\varepsilon} \exp\left\{ \lambda [Z(f) - Z(g)] \right\}
    \le \exp\left\{ \frac{\lambda^2 \sum_{i=1}^n a_i^2}{2n^3} \right\}.
  \end{align}
  Note that, by the Cauchy-Schwarz inequality,
  \begin{align*}
    a_i^2 \le \left[ \sum_{j=1}^n (\varepsilon_j')^2 \right] \left[ \sum_{j=1}^n [f(X_i, Y_j) - g(X_i, Y_j)]^2 \right] = n\left[ \sum_{j=1}^n [f(X_i, Y_j) - g(X_i, Y_j)]^2 \right].
  \end{align*}
  This yields that
  \begin{align}
    \Expect_{\varepsilon} \exp\left\{ \lambda [Z(f) - Z(g)] \right\}
    \le \exp\left\{ \frac{\lambda^2 \sum_{i,j=1}^n [f(X_i, Y_j) - g(X_i, Y_j)]^2}{2n^2} \right\}
    = \exp\left\{ \frac{\lambda^2 \norm{f - g}_{\ltwo(\hat P_n \otimes \hat Q_n)}^2}{2} \right\},
  \end{align}
  and thus the claim follows.
  Therefore, the conclusion in \Cref{prop:prod_emp_bound} holds true due to \Cref{prop:metric_entropy_integral}.
\end{proof}

\begin{proof}[Proof of \Cref{prop:cover_number_prod}]
  The proof of the first part is similar to \Cref{prop:cover_number}.
  Define $L_1 := \hpx [e^{\norm{X}^2/2d\sigma^2}]$ and $L_2 := \hpy [e^{\norm{Y}^2/2d\sigma^2}]$.
  It is clear that $L_1 \ge 1$ and $L_2 \ge 1$.
  Moreover, it follows from the sub-Gaussian assumption that $\Expect[L_1] \le 2$ and $\Expect[L_2] \le 2$.
  There are two places in the proof of \Cref{prop:cover_number} where the measure is involved.
  The first place is \eqref{eq:assign_prob_partition}, where we replace it by
  \begin{align*}
      (\hprodxy)\{(X, Y) \in B_{j,k}\}
      &\le (\hprodxy)\left\{ \norm{X}^2 + \norm{Y}^2 > d\sigma^2(j-1)^2 \right\} \\
      &\le (\hprodxy) \left[ \exp\left( \frac{\norm{X}^2 + \norm{Y}^2}{4d\sigma^2} \right) \right] e^{-(j-1)^2/4}, \quad \mbox{by Chernoff bound} \\
      &= L_1 L_2 e^{-(j-1)^2/4}.
  \end{align*}
  The second place is \eqref{eq:max_norm_bound}, where we replace it by
  \begin{align*}
      \max_{f \in \calF^s} \norm{f}^2_{\ltwo(\hprodxy)} = \max_{f \in \calF^s} (\hprodxy)[ \abs{f(X, Y)}^2] \le C_{s,d} (\hprodxy)[1 + \norm{X}^4 + \norm{Y}^4].
  \end{align*}
  Note that $\norm{Z}^4 \le C_d e^{\norm{Z}^2/2d\sigma^2} \sigma^4$. It follows that $(\hprodxy)[\norm{X}^4 + \norm{Y}^4] \le C_d(L_1 + L_2) \sigma^4$.
  Hence, the claim holds true for $L := (L_1 + L_2)/2$.
  
  For the second part, we define $\theta_f := \Expect_{\prodxy}[f(X, Y)]$,
  \begin{align}
    f_{1,0}(X) := \Expect_{\prodxy}[f(X, Y) \mid X] \quad \mbox{and} \quad f_{0,1}(Y) := \Expect_{\prodxy}[f(X, Y) \mid Y]
  \end{align}
  for each $f \in \calF^s$.
  As a result, $\bar f(x, y) := f(x, y) - f_{1,0}(x) - f_{0,1}(y) + \theta_f$ satisfies
  \begin{align}\label{eq:degeneracy}
    \Expect_{\prodxy}[\bar f(X, Y) \mid X] \txtover{a.s.}{=} 0 \txtover{a.s.}{=} \Expect_{\prodxy}[\bar f(X, Y) \mid Y].
  \end{align}
  Note that
  \begin{align}
    &\quad \Expect\norm{\hprodxy - \prodxy}_{\calF^s}^2 \nonumber \\
    &= \Expect\left[ \sup_{f \in \calF^s} \abs{\frac1{n^2} \sum_{i,j=1}^n \big(f(X_i, Y_j) - \theta_f\big)}^2 \right] \nonumber \\
    &\le C\Expect\left[ \sup_{f \in \calF^s} \abs{\frac1{n^2} \sum_{i,j=1}^n \bar f(X_i, Y_j)}^2 + \sup_{f \in \calF^s} \abs{\frac1n \sum_{i=1}^n f_{1,0}(X_i) - \theta_f}^2 + \sup_{f \in \calF^s} \abs{\frac1n \sum_{i=1}^n f_{0,1}(Y_i) - \theta_f}^2 \right] \nonumber \\
    &\le C\Expect\left[ \sup_{f \in \calF^s} \abs{\frac1{n^2} \sum_{i,j=1}^n \bar f(X_i, Y_j)}^2 + \norm{\hpx - \px}_{\calF^{s}_\sigma}^2 + \norm{\hpy - \py}_{\calF^{s}_\sigma}^2 \right], \quad \mbox{by \Cref{lem:subg_func_marginal}}. \label{eq:emp_prod_measure}
  \end{align}
  Since the last two terms above can be controlled by \Cref{prop:emp_bound}, it remains to consider the first term.
  Analogous to the proof of \Cref{prop:emp_bound}, we obtain, by \Cref{prop:prod_emp_bound} and the first part, that
  \begin{align*}
      \Expect\left[ \sup_{f \in \calF^s} \abs{\frac1{n^2} \sum_{i,j=1}^n \bar f(X_i, Y_j)}^2 \right] \le C_{s,d} (1 + \sigma^{2d+4}) \frac1n.
  \end{align*}
  Therefore, by \eqref{eq:emp_prod_measure}, we have
  \begin{align*}
      \Expect\norm{\hat P_n \otimes \hat Q_n - P \otimes Q}_{\calF^s}^2 \le C_{s,d} (1 + \sigma^{2d+4}) \frac1n.
  \end{align*}
\end{proof}

Now we are ready to prove \Cref{thm:consistency}.
\begin{proof}[Proof of \Cref{thm:consistency}]
  We prove the statement for $\varepsilon = 1$ and write $S := S_1$.
  The result for general $\varepsilon > 0$ follows immediately from \Cref{lem:reg_one}.
  By the triangle inequality, it holds that
  \begin{align}
    \abs{\heotic(X, Y) - \eotic(X, Y)}
    &\le \abs{S(\hpxy, \hprodxy) - S(\pxy, \prodxy)} + \frac12\abs{S(\hpxy, \hpxy) - S(\pxy, \pxy)} \nonumber \\
    &\quad + \frac12 \abs{S(\hprodxy, \hprodxy) - S(\prodxy, \px \otimes \py)}. \label{eq:eotic_null_triangle}
  \end{align}
  We begin with deriving the bound for the first term
  \begin{align}
    A := \abs{S(\hpxy, \hprodxy) - S(\pxy, \prodxy)}.
  \end{align}
  
  \emph{Step 1. Upper bound via empirical processes.}
  According to \Cref{lem:emp_subg} and \Cref{lem:subg_joint}, the joint distribution $\pxy$ is subG$(2\sigma^2)$, and thus there exist a zero-measure set $S_{\pxy} \subset \Omega$ and a random variable $\sigma_{\pxy}^2$ such that $\hpxy(\omega)$ and $\pxy$ are subG$(\sigma_{\pxy}^2(\omega))$ for every $\omega \in S_{\pxy}^c$.
  Similarly, by \Cref{lem:prod_emp_subg}, there exist a zero-measure set $S_{\px, \py} \subset \Omega$ and a random variable $\sigma_{\px, \py}^2$ such that $\hpx(\omega) \otimes \hpy(\omega)$ and $\prodxy$ are subG$(\sigma_{P, Q}^2(\omega))$ for every $\omega \in S_{\px, \py}^c$.
  Take $S := S_{\pxy}^c \cap S_{\px, \py}^c$ and $\bar \sigma^2 := \max\{\sigma_{\pxy}^2, \sigma_{\px, \py}^2\}$.
  It follows that $\hpxy(\omega)$, $\hpx(\omega) \otimes \hpy(\omega)$, $\pxy$, and $\prodxy$ are subG$(\bar \sigma^2(\omega))$ for every $\omega \in S$.
  Now, by \Cref{prop:upper_bound_emp},
  \begin{align*}
    &\quad \abs{S(\hpxy(\omega), \hpx(\omega) \otimes \hpy(\omega)) - S(\pxy, \prodxy)} \\
    &\le \sup_{f \in \calF_{\bar \sigma(\omega)}} \abs{\int f (d\hpxy(\omega) -  d\pxy)} + \sup_{g \in \calF_{\bar \sigma(\omega)}} \abs{\int g(d\hpx(\omega) \otimes \hpy(\omega) - d\prodxy)}, \quad \forall \omega \in S.
  \end{align*}
  Note that $\Prob(S) = \Prob(S_{\pxy}^c \cap S_{\px, \py}^c) = 1$.
  This implies, almost surely,
  \begin{align}
    A \le \sup_{f \in \calF_{\bar \sigma}} \abs{\int f (d\hpxy -  d\pxy)} + \sup_{g \in \calF_{\bar \sigma}} \abs{\int g(d\hprodxy - d\prodxy)}.
  \end{align}
  According to \Cref{lem:subg_func}, we have
  \begin{align*}
    \Expect[A]
    &\le \Expect\left[ (1 + \bar \sigma^{3s}) \norm{\hpxy - \pxy}_{\calF^s} \right] + \Expect\left[ (1 + \bar \sigma^{3s}) \norm{\hprodxy - \prodxy}_{\calF^s} \right] \\
    &\le \sqrt{\Expect[(1 + \bar \sigma^{3s})^2]} \left[ \sqrt{\Expect \norm{\hpxy - \pxy}_{\calF^s}^2} + \sqrt{\Expect \norm{\hprodxy - \prodxy}_{\calF^s}^2} \right].
  \end{align*}

  \emph{Step 2. Control empirical processes via metric entropy.}
  Let $s = \lceil d/2 \rceil + 1$.
  Since the joint probability $P_{XY}$ is subG$(2\sigma^2)$, it follows from \Cref{prop:emp_bound} that
  \begin{align}
      \sqrt{\Expect \norm{\hpxy - \pxy}_{\calF^s}^2} \le C_d (1 + \sigma^{d+2}) \frac1{\sqrt{n}}.
  \end{align}
  The same bound holds for $\sqrt{\Expect \norm{\hprodxy - \prodxy}_{\calF^s}^2}$ by \Cref{prop:prod_emp_bound}.
  Note that
  \begin{align*}
      \Expect[(1 + \tilde \sigma^{3s})^2] \le C (1 + \Expect\tilde \sigma^{6s}) \le C_s (1 + \Expect \sigma_{P_{XY}}^{6s} + \Expect \sigma_{P_{X}, P_Y}^{6s}) \le C_s (1 + \sigma^{6s}),
  \end{align*}
  where the last inequality follows from \Cref{lem:emp_subg} and \Cref{lem:prod_emp_subg}.
  Recall that we have chosen $s = \lceil d/2 \rceil + 1$.
  As a result, $\Expect[A] \le C_d (1 + \sigma^{\lceil 5d/2 \rceil + 6}) n^{-1/2}$.
  A similar argument shows that the same bound hold for the second and third term in \eqref{eq:eotic_null_triangle}.
  Hence,
  \begin{align}
      \Expect\abs{\heotic(X, Y)} \le C_d (1 + \sigma^{\lceil 5d/2 \rceil + 6}) \frac1{\sqrt{n}}.
  \end{align}
\end{proof}

%%%%%%%%%%%%%%%%%%%%%%%%%%%%%%%%%%%%%%%%%%%%%%%%%%%%%%%%%%%%%%%%%%%%%%%%
%%%%%%%%%%%%%%%%%%%%%%%%%%%%%%%%%%%%%%%%%%%%%%%%%%%%%%%%%%%%%%%%%%%%%%%%
\section{Exponential Tail Bounds}
\label{sec:tail_bounds}
%%%%%%%%%%%%%%%%%%%%%%%%%%%%%%%%%%%%%%%%%%%%%%%%%%%%%%%%%%%%%%%%%%%%%%%%
%%%%%%%%%%%%%%%%%%%%%%%%%%%%%%%%%%%%%%%%%%%%%%%%%%%%%%%%%%%%%%%%%%%%%%%%

We now prove the exponential tail bound in \Cref{sec:thm}.
For simplicity of the notation, we focus on the quadratic cost function, i.e., $w_1 = w_2 = 1$, and drop the dependency on $w$ (e.g., we write $C_{s,d} = C_{s, d, w}$.
The proofs can be adapted to weighted quadratic costs with minor modifications.
Let $\px \in \pspace(\reals^{d_1})$ and $\py \in \pspace(\reals^{d_2})$ with $d := d_1 + d_2$.
Suppose that $\{(X_i, Y_i)\}_{i=1}^n$ is an i.i.d.~sample from some joint distribution $\pxy$ with marginals $\px$ and $\py$, where $\pxy$ may or may not equal $\prodxy$.
Let $\hat P_n$ and $\hat Q_n$ be the empirical measures of $\{X_i\}_{i=1}^n$ and $\{Y_i\}_{i=1}^n$, respectively.

\begin{proposition}\label{prop:prod_tail_bound}
  For any $b$-uniformly bounded class of functions $\calF$, we have
  \begin{align*}
      \Prob\left\{ \norm{\hprodxy - \prodxy}_{\calF} - \Expect\norm{\hprodxy - \prodxy}_{\calF} > t \right\} \le \exp\left( -\frac{nt^2}{8b^2} \right), \quad \mbox{for any } t \ge 0.
  \end{align*}
\end{proposition}
\begin{proof}
  For any function $f$ defined on $\reals^d$, we define $\bar f(x, y) = f(x, y) - (\prodxy)[f]$.
  As a results, we have $\norm{\hprodxy - \prodxy}_{\calF} = \sup_{f \in \calF} \abs{\frac1{n^2} \sum_{i,j=1}^n \bar f(X_i, Y_j)}$.
  Consider the function
  \begin{align}
      F(z_1, \dots, z_n) := \sup_{f \in \calF} \abs{\frac1{n^2} \sum_{i,j=1}^n \bar f(x_i, y_j)},
  \end{align}
  where $z_i = (x_i, y_i) \in \reals^d$.
  We claim that $F$ satisfies the bounded difference property required in the McDiarmid inequality.
  Since $F$ is permutation invariant, it suffices to verify the property for the first coordinate.
  Let $z_1' \neq z_1$ and $z_i' = z_i$ for all $i \neq 1$.
  It holds that
  \begin{align*}
      \abs{\frac1{n^2} \sum_{i,j=1}^n \bar f(x_i, y_j)} - F(z_1', \dots, z_n')
      &\le \abs{\frac1{n^2} \sum_{i,j=1}^n \bar f(x_i, y_j)} - \abs{\frac1{n^2} \sum_{i,j=1}^n \bar f(x_i', y_j')} \\
      &\le \frac1{n^2} \sum_{i=1 \text{ or } j=1} \abs{\bar f(x_i, y_j) - \bar f(x_i', y_j')} \le \frac{4b}{n},
  \end{align*}
  where the last inequality uses the boundedness of $f$.
  Taking the supremum over $\calF$ yields that $F(z_1, \dots, z_n) - F(z_1', \dots, z_n') \le 4b/n$.
  By symmetry, it follows that $\abs{F(z_1, \dots, z_n) - F(z_1', \dots, z_n')} \le 4b/n$.
  Note that $\{Z_i := (X_i, Y_i)\}_{i=1}^n$ is an i.i.d.~sample.
  According to the McDiarmid inequality, it holds that
  \begin{align*}
      \Prob\left\{ \norm{\hprodxy - \prodxy}_{\calF} - \Expect\norm{\hprodxy - \prodxy}_{\calF} > t \right\} \le \exp\left( -\frac{nt^2}{8b^2} \right), \quad \mbox{for any } t \ge 0.
  \end{align*}
\end{proof}

\begin{proof}[Proof of \Cref{thm:tail_bound}]
  We prove the statement for $\varepsilon = 1$ and write $S := S_1$.
  The result for general $\varepsilon > 0$ follows immediately from \Cref{lem:reg_one}.
  By the bounded support assumption, it holds that $\px$ and $\py$ are both subG$(D^2/d)$.
  According to the proof of \Cref{lem:emp_subg}, we have $\{\hpx\}_{n \ge 1}$, $\{\hpy\}_{n \ge 1}$, $\px$, and $\py$ are uniformly subG$(\tau^2)$ for $\tau^2 := D^2 e^{1/2} / d \le 2D^2/d$.
  Moreover, it follows from \Cref{lem:subg_joint} that $\{\hpxy\}_{n\ge1}$ and $\pxy$ are uniformly subG$(2\tau^2)$.
  As a result, we obtain, by \Cref{prop:upper_bound_emp},
  \begin{align*}
      A &:= \abs{S(\hpxy, \hprodxy) - S(\pxy, \prodxy)} \\
      &\le \sup_{f \in \calF_{2\tau}} \abs{\int f (d\hpxy - d\pxy)} + \sup_{g \in \calF_{2\tau}} \abs{\int g (d\hprodxy - d\prodxy)}.
  \end{align*}
  Fix $s = \lceil d/2 \rceil + 1$.
  According to \Cref{lem:subg_func}, we have
  \begin{align}\label{eq:upper_bound_A}
      A \le C_d (1 + D^{3d+12}) \left[ \norm{\hpxy - \pxy}_{\calF^s} + \norm{\hprodxy - \prodxy}_{\calF^s} \right],
  \end{align}
  where we have used $\tau^{3s} \le C_d D^{3d+12}$.
  \Cref{prop:bounded_potential} shows that we can further constraint the function class $\calF^s$ to $\calF^s_b := \{f \in \calF^s: \norm{f}_\infty \le b\}$ for $b = 2D^2$.
  Hence, by \cite[Theorem 4.10]{wainwright2019high}, it holds that
  \begin{align*}
      \Prob\left\{ \norm{\hpxy - \pxy}_{\calF^s_b} - \Expect\norm{\hpxy - \pxy}_{\calF^s_b} > t \right\} \le \exp\left( -\frac{nt^2}{2b^2} \right), \quad \mbox{for any } t \ge 0.
  \end{align*}
  It is clear from \Cref{prop:emp_bound} that
  \begin{align*}
      \Expect\norm{\hpxy - \pxy}_{\calF^s_b} \le \Expect\norm{\hpxy - \pxy}_{\calF^s} \le C_{d} (1 + D^{2d+4}) \frac1{\sqrt{n}}.
  \end{align*}
  Consequently, we get
  \begin{align*}
      \Prob\left\{ \norm{\hpxy - \pxy}_{\calF^s_b} > t + C_{d} (1 + D^{2d+4}) \frac1{\sqrt{n}} \right\} \le \exp\left( -\frac{nt^2}{2b^2} \right), \quad \mbox{for any } t \ge 0.
  \end{align*}
  Similarly, using \Cref{prop:prod_emp_bound} and \Cref{prop:prod_tail_bound}, we obtain
  \begin{align*}
      \Prob\left\{ \norm{\hpxy - \pxy}_{\calF^s_b} > t + C_{d} (1 + D^{2d+4}) \frac1{\sqrt{n}} \right\} \le \exp\left( -\frac{nt^2}{8b^2} \right), \quad \mbox{for any } t \ge 0.
  \end{align*}
  Now it follows from \eqref{eq:upper_bound_A} that
  \begin{align*}
      \Prob\left\{ A \ge C_d (1 + D^{3d+12})\left[ t + (1 + D^{2d+4}) \frac1{\sqrt{n}} \right] \right\} \le 2\exp\left( -\frac{nt^2}{8b^2} \right), \quad \mbox{for any } t \ge 0.
  \end{align*}
  Analogously, we have, for any $t \ge 0$
  \begin{align*}
      \Prob\left\{ B \ge C_d (1 + D^{3d+12})\left[ t + (1 + D^{2d+4}) \frac1{\sqrt{n}} \right] \right\} &\le 2\exp\left( -\frac{nt^2}{8b^2} \right) \\
      \Prob\left\{ B' \ge C_d (1 + D^{3d+12})\left[ t + (1 + D^{2d+4}) \frac1{\sqrt{n}} \right] \right\} &\le 2\exp\left( -\frac{nt^2}{8b^2} \right),
  \end{align*}
  where $B := \abs{S(\hpxy, \hpxy) - S(\pxy, \pxy)}$ and $B' := \abs{S(\hprodxy, \hprodxy) - S(\prodxy, \prodxy)}$.
  Since $\abs{\heotic(X, Y) - \eotic(X, Y)} \le A + \frac{B}{2} + \frac{B'}{2}$, it holds that
  \begin{align}\label{eq:tail_bound_eotic}
      \Prob\left\{ \abs{\heotic(X, Y) - \eotic(X, Y)} \ge C_d (1 + D^{3d+12})\left[ t + (1 + D^{2d+4}) \frac1{\sqrt{n}} \right] \right\} &\le 6\exp\left( -\frac{nt^2}{8b^2} \right).
  \end{align}
  Therefore, we have, with probability at least $1 - \delta$,
  \begin{align*}
      \abs{\heotic(X, Y) - \eotic(X, Y)} \le C_d \left( 1 + D^{2d+2} \sqrt{\log{\frac6\delta}} \right) \frac{D^{3d+14}}{\sqrt{n}}.
  \end{align*}
\end{proof}

%%%%%%%%%%%%%%%%%%%%%%%%%%%%%%%%%%%%%%%%%%%%%%%%%%%%%%%%%%%%%%%%%%%%%%%%
%%%%%%%%%%%%%%%%%%%%%%%%%%%%%%%%%%%%%%%%%%%%%%%%%%%%%%%%%%%%%%%%%%%%%%%%
\section{Technical Lemmas}
\label{sec:Technical_Lemmas}
%%%%%%%%%%%%%%%%%%%%%%%%%%%%%%%%%%%%%%%%%%%%%%%%%%%%%%%%%%%%%%%%%%%%%%%%
%%%%%%%%%%%%%%%%%%%%%%%%%%%%%%%%%%%%%%%%%%%%%%%%%%%%%%%%%%%%%%%%%%%%%%%%

In this section, we give several technical lemmas used to prove the main results.
We use $C$ to denote a constant whose value may change from line to line.

\begin{lemma}\label{lem:subg_moment}
  If $P \in \pspace(\reals^d)$ is subG$(\sigma^2)$, then, for any $k \in \bbN_+$,
  \begin{align*}
    \Expect_P[\norm{Z}^{2k}] \le (2d\sigma^2)^k k!.
  \end{align*}
  Moreover, for any $v \in \reals^d$, it holds that
  \begin{align}\label{eq:a_subg_inner_prod}
      \Expect_P e^{\ip{v, Z}} \le \Expect_P e^{\norm{v} \norm{Z}} \le 2 e^{d\sigma^2 \norm{v}^2/2}.
  \end{align}
\end{lemma}
\begin{proof}
  By Taylor's expansion, we have
  \begin{align*}
      e^{\norm{Z}^2/2d\sigma^2} - 1 \ge \frac{\norm{Z}^{2k}}{(2d \sigma^2)^k k!}.
  \end{align*}
  Taking the expectation on both sides gives
  \begin{align*}
      \Expect_P[\norm{Z}^{2k}] \le (2d \sigma^2)^k k!.
  \end{align*}
  The inequalities \eqref{eq:a_subg_inner_prod} follows from the Cauchy-Schwarz inequality and the sub-Gaussianity of $P$.
\end{proof}

\begin{lemma}\label{lem:emp_subg}
  Let $P \in \pspace(\reals^d)$ be subG$(\sigma^2)$ and $\hat P_n$ be the empirical measure.
  There exist a zero-measure set $S_P \subset \Omega$ and a random variable $\sigma_P^2$ depending on the sample $\{Z_i\}_{i=1}^n$ such that $\hat P_n(\omega)$ and $P$ are subG$(\sigma_P^2(\omega))$ for any $\omega \in S_P^c$, and, for any $k \in \bbN_+$,
  \begin{align*}
    \Expect \sigma_P^{2k} \le 2k^k \sigma^{2k}.
  \end{align*}
\end{lemma}
\begin{proof}
  By the strong law of large numbers, there exists a zero-measure set $S_P \subset \Omega$ such that, for all $\omega \in S_P$,
  \begin{align}\label{eq:subg_converge}
      \hat P_n(\omega)\left[ e^{\norm{Z}^2/2d\sigma^2} \right] \rightarrow P \left[ e^{\norm{Z}^2/2d\sigma^2} \right] \le 2, \quad \mbox{as } n \rightarrow \infty.
  \end{align}
  Let $\tau^2 := \sup_{n} \hat P_n \left[ e^{\norm{Z}^2/2d\sigma^2} \right]$.
  It follows from \eqref{eq:subg_converge} that $\tau^2(\omega)$ is finite for all $\omega \in S_P$.
  Since $\tau^2(\omega) \ge 1$, by Jensen's inequality, we obtain, for all $\omega \in S_P$
  \begin{align*}
      \hat P_n(\omega) \left[ e^{\norm{Z}^2/2d\sigma^2 \tau^2(\omega)} \right] \le \left( \hat P_n(\omega) \left[ e^{\norm{Z}^2/2d\sigma^2} \right] \right)^{1/\tau^2(\omega)} = \left( \tau^2(\omega) \right)^{1/\tau^2(\omega)} < 2.
  \end{align*}
  As a result, $\hat P_n(\omega)$ is subG$(\sigma^2 \tau^2(\omega))$.
  Moreover, $P$ is also subG$(\sigma^2 \tau^2(\omega))$ since $\tau^2(\omega) \ge 1$.
  Applying the same argument to $\tau_k^2 := \sup_{n} \hat P_n \left[ e^{\norm{Z}^2/2kd\sigma^2} \right]$ implies that $\hat P_n(\omega)$ and $P$ are both subG$(k\sigma^2 \tau_k^2(\omega))$.
  Define $\sigma_P^2 := \min_{k \ge 1} k\sigma^2 \tau_k^2$.
  Then we have, for each $k \ge 1$,
  \begin{align*}
      \Expect_P [\sigma_P^{2k}] \le \Expect_P \left[ \hat P_n \left[ k^k \sigma^{2k} e^{\norm{Z}^2/2d\sigma^2} \right] \right] = k^k \sigma^{2k} \Expect_P[e^{\norm{Z}^2/2d\sigma^2}] \le 2 k^k \sigma^{2k}.
  \end{align*}
\end{proof}

The sub-Gaussianity of two marginals implies the sub-Gaussianity of the joint.
\begin{lemma}\label{lem:subg_joint}
  If $\px$ and $\py$ are subG$(\sigma^2)$, then $\pxy$ is subG$(2\sigma^2)$ for any $\pxy \in \Pi(\px, \py)$.
\end{lemma}
\begin{proof}
  By the Cauchy-Schwarz inequality,
  \begin{align*}
      \Expect_{\pxy} e^{\norm{Z}^2/4d\sigma^2} = \Expect_{\pxy} [e^{\norm{X}^2/4d\sigma^2} e^{\norm{Y}^2/4d\sigma^2}] \le \sqrt{\Expect_{\px}[e^{\norm{X}^2/2d\sigma^2}] \Expect_{\py}[e^{\norm{Y}^2/2d\sigma^2}]}.
  \end{align*}
  Since $\px$ and $\py$ are subG$(\sigma^2)$, it follows that $\Expect_{\pxy} e^{\norm{Z}^2/4d\sigma^2} \le 2$ and thus $\pxy$ is subG$(2\sigma^2)$.
\end{proof}

The next result is for the uniform sub-Gaussianity of the product of two empirical measures.
\begin{lemma}\label{lem:prod_emp_subg}
  If $\px$ and $\py$ are subG$(\sigma^2)$,
  then there exist a zero-measure set $S_{\px, \py} \subset \Omega$ and a random variable $\sigma_{\px, \py}^2$ depending on the sample $\{(X_i, Y_i)\}_{i=1}^n$ such that $\hpx(\omega) \otimes \hpy(\omega)$ and $\prodxy$ are subG$(\sigma_{\px, \py}^2(\omega))$ for any $\omega \in S_{\px, \py}^c$, and, for any $k \in \bbN_+$,
  \begin{align*}
    \Expect \sigma_{\px, \py}^{2k} \le 2^{k+1}k^k \sigma^{2k}.
  \end{align*}
\end{lemma}
\begin{proof}
  Similar to \Cref{lem:emp_subg}.
\end{proof}

The sub-Gaussian processes play an central role in our analysis.
We give its definition here; see, e.g., \cite[Section 5.3]{wainwright2019high}.
\begin{definition}[Sub-Gaussian process]\label{def:subg_process}
  Let $\{Z(\theta): \theta \in \Theta\}$ be a collection of mean-zero random variables.
  We call it a sub-Gaussian process with respect to a metric $\rho$ in $\Theta$ if
  \begin{align*}
    \Expect[e^{\lambda (Z(\theta) - Z(\theta'))}] \le \exp\left[ \lambda^2 \rho^2(\theta, \theta')/2 \right].
  \end{align*}
\end{definition}

To facilitate the analysis of $\calF_{\sigma}$ defined in \Cref{def:f_sigma}, it is convenient to separate the sub-Gaussian parameter from the function class by the following lemma.
Note that this result is used in \citep{mena2019statistical} without proof.
\begin{lemma}\label{lem:subg_func}
  For any $\sigma > 0$ and $s \ge 2$.
  we have $\frac1{1 + \sigma^{3s}} \calF_\sigma \subset \calF^s$, where $\calF^s := \calF^{s,d,w}$ is defined in \Cref{def:fs}.
\end{lemma}
\begin{proof}
  Take any $f \in \calF_\sigma$, it suffices to show $f/(1 + \sigma^{3s}) \in \calF^s$.
  According to \Cref{prop:smooth_potentials}, it holds that
  \begin{align*}
      \abs{f(z)} - w_1 \norm{x}^2 - w_2\norm{y}^2
      \le \abs{f(z) - w_1 \norm{x}^2 - w_2\norm{y}^2} \le C_{k, d, w}
      \begin{cases}
        (1 + \sigma^4) & \text{if } \norm{z} \le \sqrt{d} \sigma \\
        [1 + (1 + \sigma^2) \norm{z}^2] & \mbox{if } \norm{z} > \sqrt{d} \sigma.
      \end{cases}
  \end{align*}
  Consequently,
  \begin{align*}
      \abs{\frac{f(z)}{1 + \sigma^{3s}}} \le C_{k, d, w}
      \begin{cases}
        \frac{1 + \sigma^4}{1 + \sigma^{3s}} & \text{if } \norm{z} \le \sqrt{d} \sigma \\
        \frac{1 + (1 + \sigma^2) \norm{z}^2}{1 + \sigma^{3s}} & \mbox{if } \norm{z} > \sqrt{d} \sigma.
      \end{cases}
  \end{align*}
  Since $s \ge 2$, it is clear that $\frac{1 + \sigma^4}{1 + \sigma^{3s}} \le C$ and $\frac{1 + \sigma^2}{1 + \sigma^{3s}} \le C$, and thus
  \begin{align*}
      \abs{\frac{f(z)}{1 + \sigma^{3s}}} \le C_{k, d, w} (1 + \norm{z}^2).
  \end{align*}
  The other inequality can be proved analogously.
\end{proof}

\begin{lemma}\label{lem:subg_func_marginal}
  Let $P \in \pspace(\reals^{d_1})$ and $Q \in \pspace(\reals^{d_2})$ be subG$(\sigma^2)$.
  Denote $d := d_1 + d_2$.
  For any $s \ge 1$ and $f \in \calF^{s}$, there exist constants $C_{s, d, w}$ such that $f_{1,0} \in \calF^{s}_\sigma$ and $f_{0,1} \in \calF^{s}_\sigma$, where $\calF_\sigma^s$ is defined in \Cref{def:fs_sigma},
  \begin{align*}
    f_{1,0}(x) := \int f(x, y) dQ(y) \quad \mbox{and} \quad f_{0,1}(y) := \int f(x, y) dP(x).
  \end{align*}
\end{lemma}
\begin{proof}
  We only prove it for $f_{1,0}$.
  By Jensen's inequality, it holds that
  \begin{align*}
      \abs{f_{1,0}(x)} \le \int \abs{f(x, y)} dQ(y) \le C_{s, d, w} \left( 1 + \norm{x}^2 + \int \norm{y}^2 dQ(y) \right) \le C_{s,d,w} (1 + \max\{\norm{x}^2, \sigma^2\}),
  \end{align*}
  where the last inequality follows from \Cref{lem:subg_moment}.
  The inequality for $\abs{D^\alpha f_{1,0}(x)}$ can be verified similarly.
\end{proof}

The next lemma suggests that it is enough to consider the case $\varepsilon = 1$ for $S_\varepsilon$.
\begin{lemma}\label{lem:reg_one}
  Let $\varepsilon > 0$.
  For any $P, Q \in \pspace(\reals^d)$, it holds that
  \begin{align*}
    S_\varepsilon(P, Q) = \varepsilon S(P^\varepsilon, Q^\varepsilon),
  \end{align*}
  where $P^\varepsilon$ and $Q^\varepsilon$ are the pushforwards of $P$ and $Q$ under the map $x \mapsto \varepsilon^{-1/2}x$, respectively.
\end{lemma}
\begin{proof}
  By a change of variable argument.
\end{proof}

%%%%%%%%%%%%%%%%%%%%%%%%%%%%%%%%%%%%%%%%%%%%%%%%%%%%%%%%%%%%%%%%%%%%%%%%
%%%%%%%%%%%%%%%%%%%%%%%%%%%%%%%%%%%%%%%%%%%%%%%%%%%%%%%%%%%%%%%%%%%%%%%%
\section{Additional Experimental Results}
\label{sec:more_results}
%%%%%%%%%%%%%%%%%%%%%%%%%%%%%%%%%%%%%%%%%%%%%%%%%%%%%%%%%%%%%%%%%%%%%%%%
%%%%%%%%%%%%%%%%%%%%%%%%%%%%%%%%%%%%%%%%%%%%%%%%%%%%%%%%%%%%%%%%%%%%%%%%

\subsection{\eoticrf~on Bilingual Text}
\label{sub:etic_rf_text}

We examine the \eoticrf~test on bilingual data discussed in \Cref{sub:real}.
The feature embeddings are of high dimension (i.e., 768), which imposes challenges on the random feature approximation.
Hence, we first use dimension reduction (PCA) on the English embeddings and French embeddings separately to reduce the dimension to $d' \ll 768$, and then perform \eoticrf~on the low-dimensional embeddings.
Since the dimension reduction step does not utilize information about the joint distribution $\pxy$, it will not violate the level consistency of the test.
This is also validated in our experimental results, i.e., all the tests have type I error rate close to $0.05$ as expected.

\begin{figure*}[t]
\centering
\includegraphics[width=0.6\textwidth]{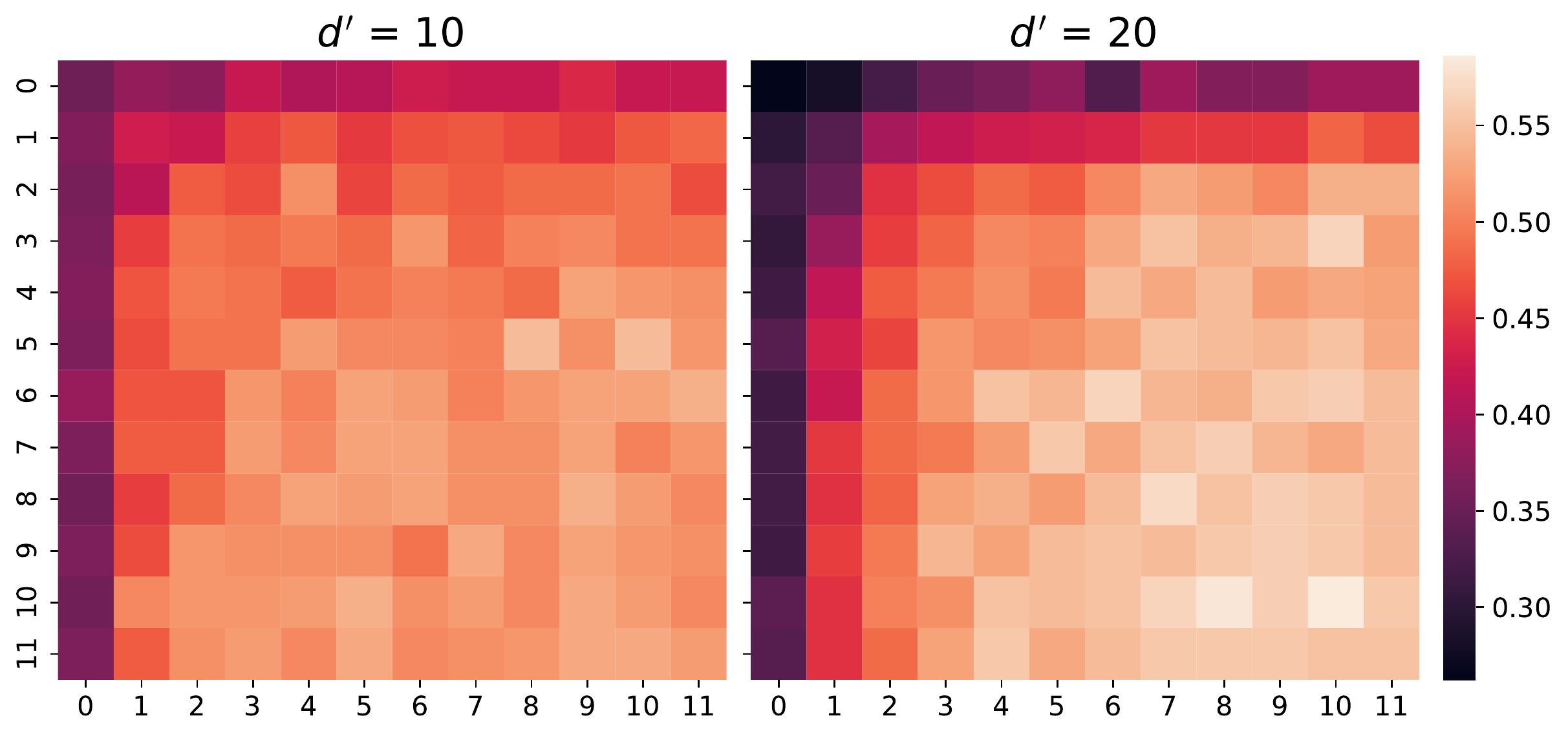}
\vspace{.1in}
\caption{Heatmaps of power for \eoticrf~with $p=700$ random features and $d'$ PCs on the partially dependent sample of the bilingual data. The $x$-axis is for $r_1$ and $y$-axis is for $r_2$. The indices from 0 to 11 correspond to equally spaced values from 0.25 to 4. Lighter color indicates larger power.}
\label{fig:rfeot_nfeat700}
\end{figure*}

\begin{figure*}[t]
\centering
\includegraphics[width=0.6\textwidth]{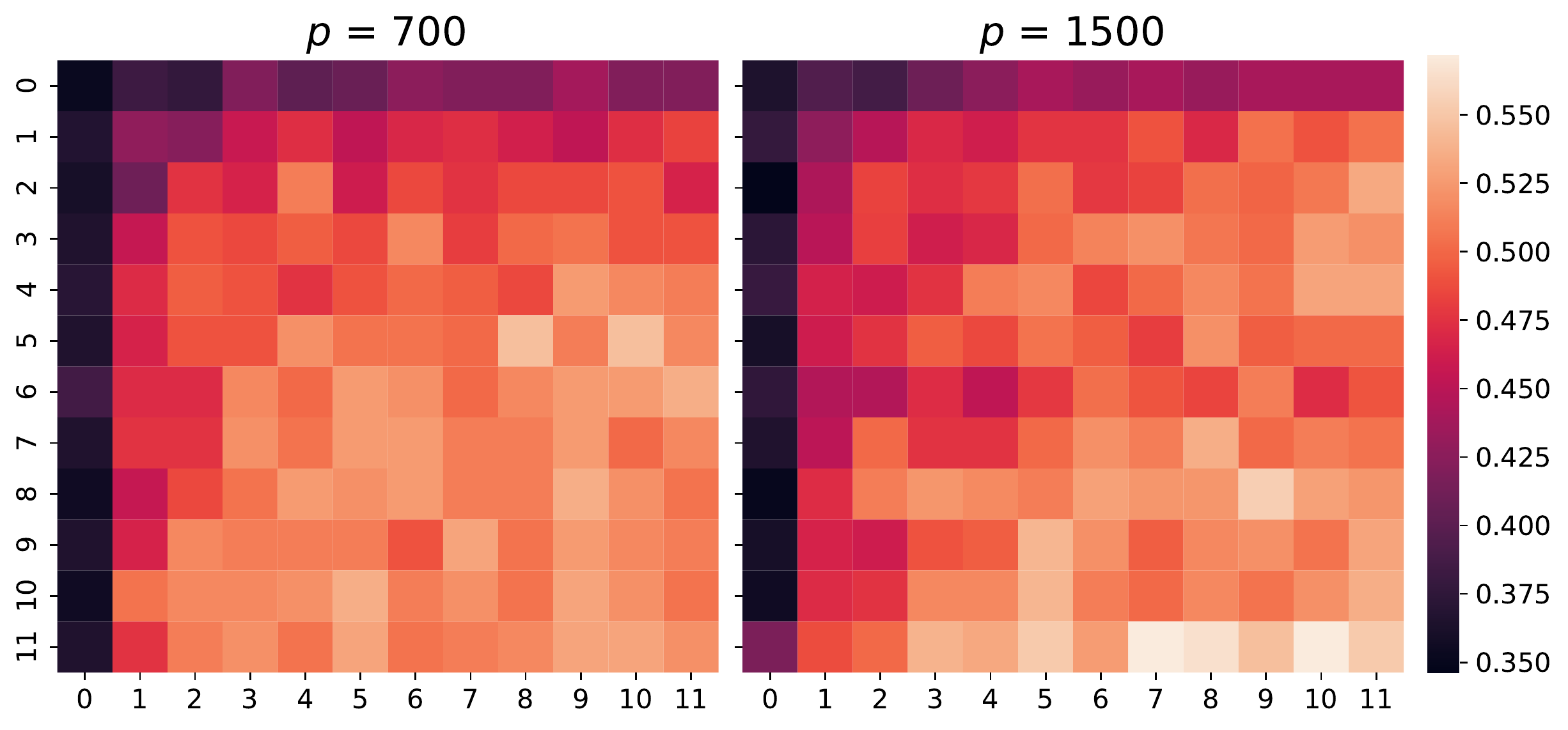}
\vspace{.1in}
\caption{Heatmaps of power for \eoticrf~with $d'=10$ PCs and $p$ random features on the partially dependent sample of the bilingual data. The $x$-axis is for $r_1$ and $y$-axis is for $r_2$. The indices from 0 to 11 correspond to equally spaced values from 0.25 to 4. Lighter color indicates larger power.}
\label{fig:rfeot_npc10}
\end{figure*}

As shown in \Cref{fig:rfeot_nfeat700}, The number of PCs $d'$ has an interesting effect on the power.
Intuitively, the larger $d'$ is the less information we lose, and thus the larger power the test has.
This can be seen at the lower right corner where both $r_1$ and $r_2$ are large.
However, larger $d'$ also means the random feature approximation is harder, especially when $r_1$ and $r_2$ are small.
This is reflected at the upper left corner where the power decreases as $d'$ increases.
We then investigate the effect of $p$---the number of random features.
As shown in \Cref{fig:rfeot_npc10}, the power increases with the number of random features. 
Overall, the random feature approximation demonstrates similar performance as the exact \teotic~with enough random features.

\begin{figure*}[t]
\centering
\includegraphics[width=0.7\textwidth]{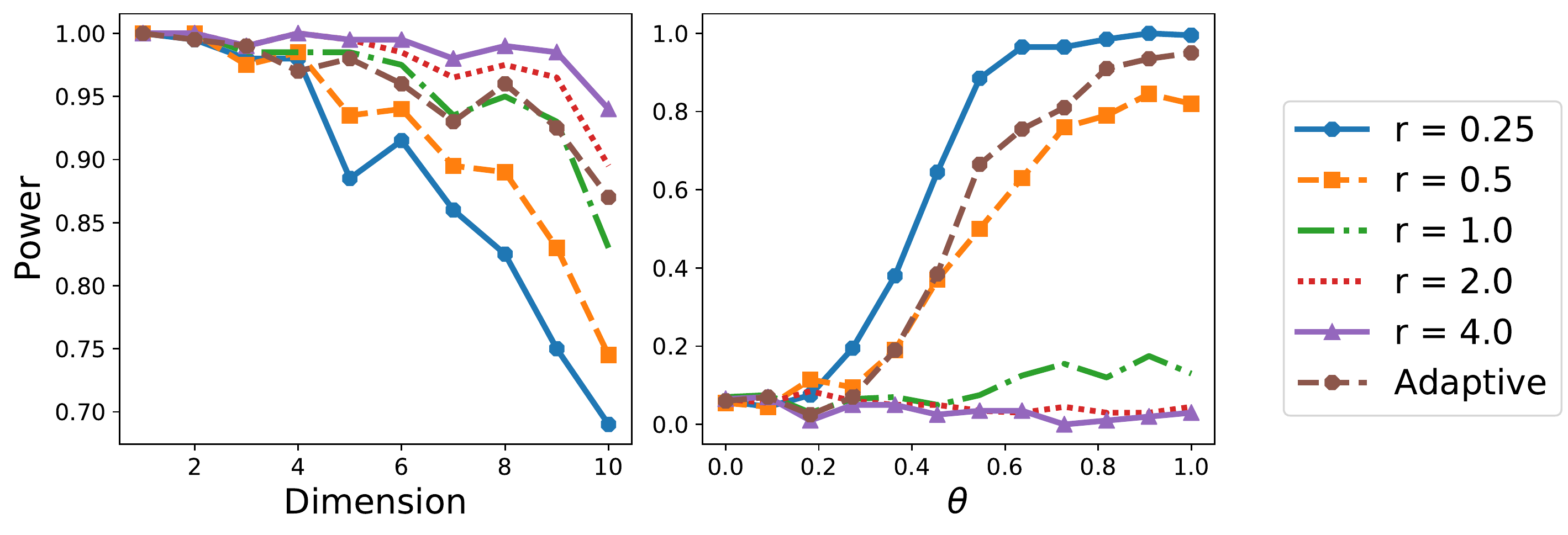}
\vspace{.1in}
\caption{Power curves in the linear dependency model (left) and subspace dependency model (right).}
\label{fig:adaptive_etic}
\end{figure*}

\begin{figure*}[t]
\centering
\includegraphics[width=0.7\textwidth]{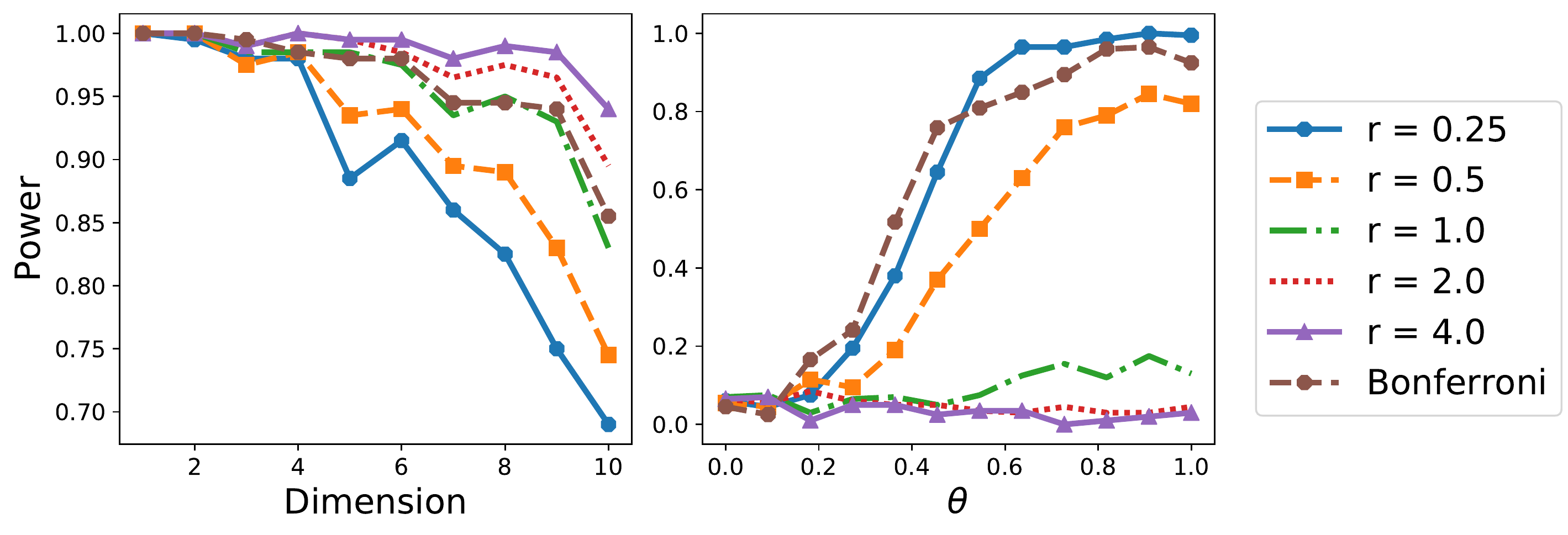}
\vspace{.1in}
\caption{Power curves in the linear dependency model (left) and subspace dependency model (right).}
\label{fig:bonferroni_etic}
\end{figure*}

\subsection{Adaptive \teotic~Test}
\label{sub:adaptive_etic}

To ensure that $(T_{n, \varepsilon}(X, Y))_\varepsilon$ are roughly on the same scale for different values of $\varepsilon$, we use
\begin{align*}
    \psi_a(\alpha) := \ind\left\{ \max_{\varepsilon \in \mathcal{E}} \bar T_{n, \varepsilon}(X, Y) > H_{n, \mathcal{E}}(\alpha) \right\}.
\end{align*}
where each $T_{n, \varepsilon}(X, Y)$ is studentized via resampling (20 permutations) under the null to yield $\bar T_{n, \varepsilon}(X, Y)$. 

Following the trick in \Cref{sec:experiments}, we select the cost function to be the weighted quadratic cost with weights given by the median heuristic.
We set $\mathcal{E} = \{0.25, 1, 4\}$ and perform the adaptive \teotic~test on the linear dependency model and the subspace dependency model.
As shown in \Cref{fig:adaptive_etic}, it is slightly worse than the best \teotic~test in both models.
We also run it on the bilingual text data.
The power and type I error rate of adaptive \teotic~are $1$ and $0.07$ on the dependent sample and the independent sample, respectively.
The power achieved is $0.535$ on the partially dependent sample; whereas the worst and best power of \teotic~are $0.38$ and $0.635$, respectively.

Finally, we consider a Bonferroni-type \teotic~test which is adaptive to both the regularization parameter and the weights in the cost function.
Following the formulation in \Cref{sec:experiments}, we let $\psi_{r_1, r_2}(\alpha)$ be the decision rule of \teotic~with hyper-parameters $r_1$ and $r_2$.
Consider the following Bonferroni-type \teotic~test
\begin{align*}
    \psi(\alpha) := \max_{r_1, r_2 \in \mathcal{R}} \psi_{r_1, r_2}(\alpha / \abs{\mathcal{R}^2}).
\end{align*}

We perform this Bonferroni-type \teotic~test on the linear dependency model and the subspace dependency model for $\mathcal{R} = \{0.25, 4\}$.
As shown in \Cref{fig:bonferroni_etic}, it is slightly worse than the best \teotic~test in both models.
Compared to the adaptive \teotic~test, it performs similar in the linear dependency model and slightly better in the subspace dependency model.
We also run it on the bilingual text data.
The power and type I error rate of adaptive \teotic~are $1$ and $0.045$ on the dependent sample and the independent sample, respectively.
The power achieved is $0.5$ on the partially dependent sample, which is smaller than the power of the adaptive \teotic.

\begin{figure*}[t]
    \centering
    \includegraphics[width=0.65\textwidth]{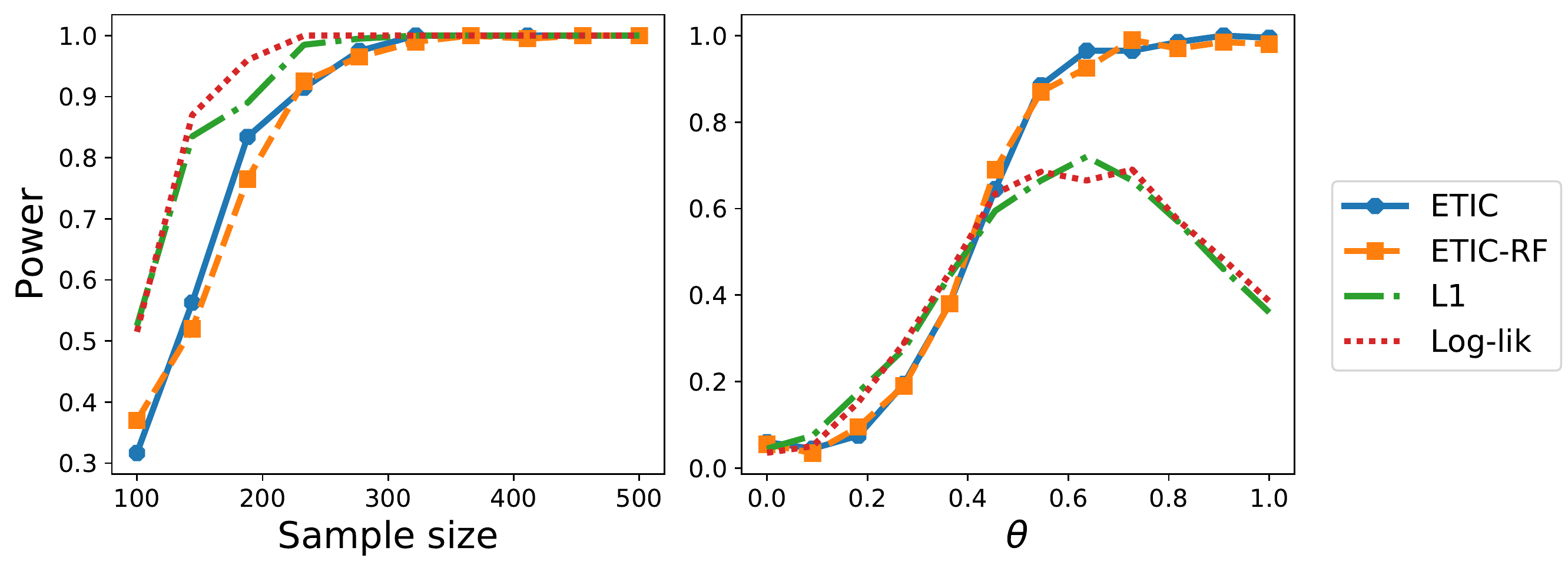}
    \vspace{.1in}
    \caption{Power curves in the Gaussian sign model (left) and subspace dependency model (right).}
    \label{fig:more_tests}
\end{figure*}

\subsection{Comparison with Other Independence Tests}
\label{sub:comparison}
We implemented another two independence tests considered in \citep{gretton2008nonparametric}: the $L_1$ test and the Log-likelihood test.
We apply them to the Gaussian sign model and the subspace dependency model and compare them with \teotic~and \eoticrf~with $r = 0.25$.
As shown in \Cref{fig:more_tests}, they slightly outperform \teotic~in the Gaussian sign model.
In the subspace dependency model, they perform similarly to \teotic~for small $\theta$ and significantly worse for large $\theta$.

\end{document}